\documentclass[twoside]{article}

\usepackage{aistats2025}
\usepackage{graphicx}
\usepackage{amsfonts}
\usepackage{amsthm}
\usepackage{enumitem}
\usepackage{url}
\newtheoremstyle{customstyle}
  {\bigskipamount}   
  {\bigskipamount}   
  {\normalfont}      
  {}                 
  {\bfseries}        
  {.}                
  {.5em}             
  {}                

\theoremstyle{customstyle}

\newtheorem{dfn}{Definition}
\newtheorem{thm}[dfn]{Theorem}
\newtheorem{lem}[dfn]{Lemma}
\newtheorem{prop}[dfn]{Proposition}

\newtheorem{ex}[dfn]{Example}
\newtheorem{cor}[dfn]{Corollary}

\begin{document}


\twocolumn[

\aistatstitle{Benign Overfitting under Learning Rate Conditions for \\$\alpha$ Sub-exponential Inputs}

\aistatsauthor{%
  \bfseries Kota Okudo \\ Graduate School of Science and Technology, \\Keio University \\ okudokota@keio.jp \and
  \bfseries Kei Kobayashi \\ Department of Mathematics, \\Keio University \\ kei@math.keio.ac.jp
} 
\vskip\baselineskip
]

\begin{abstract}
    This paper investigates the phenomenon of benign overfitting in binary classification problems with heavy-tailed input distributions, extending the analysis of maximum margin classifiers to \(\alpha\) sub-exponential distributions (\(\alpha \in (0, 2]\)). This generalizes previous work focused on sub-gaussian inputs. We provide generalization error bounds for linear classifiers trained using gradient descent on unregularized logistic loss in this heavy-tailed setting. Our results show that, under certain conditions on the dimensionality \(p\) and the distance between the centers of the distributions, the misclassification error of the maximum margin classifier asymptotically approaches the noise level, the theoretical optimal value. Moreover, we derive an upper bound on the learning rate \(\beta\) for benign overfitting to occur and show that as the tail heaviness of the input distribution \(\alpha\) increases, the upper bound on the learning rate decreases. These results demonstrate that benign overfitting persists even in settings with heavier-tailed inputs than previously studied, contributing to a deeper understanding of the phenomenon in more realistic data environments.
\end{abstract}
\section{Introduction}\label{section:introduction}
In the field of machine learning, a phenomenon that contradicts the long-standing intuition of statistical learning theory has been garnering attention. This phenomenon is called benign overfitting. According to conventional theory, when a model excessively fits the training data, its generalization performance on unseen data was expected to decline. However, experiments using deep neural networks have revealed that models that perfectly fit noisy training data surprisingly demonstrate good performance on unseen data as well \cite{DBLP:journals/corr/ZhangBHRV16, Belkin_2019}.

This phenomenon suggests a significant gap between machine learning theory and practice, attracting the attention of researchers. To deepen our understanding of benign overfitting, studies have been conducted in simpler statistical settings that are more amenable to theoretical analysis, such as linear regression \cite{hastie2022surprises, bartlett2020benign, muthukumar2020harmless, negrea2020defense, tsigler2023benign, chinot2022robustness, chatterji2022interplay}, sparse linear regression \cite{koehler2021uniform, chatterji2022foolish, li2021minimum, wang2022tight}, logistic regression \cite{montanari2019generalization, JMLR:v22:20-974, liang2022precise, JMLR:v22:20-603, wang2021benign, minsker2021minimax, frei2022benign, zhu2023benignoverfittingdeepneural}, and kernel-based estimators \cite{belkin2018overfitting, mei2022generalization, Liang_2020, liang2020multiple}. These studies are rapidly advancing our understanding of the conditions and mechanisms under which benign overfitting occurs.

In the context of binary classification, a standard mixture model is often used to study benign overfitting \cite{JMLR:v22:20-974, frei2022benign, zhu2023benignoverfittingdeepneural}. This model involves classifying well-separated data with adversarially corrupted labels, assuming the input distribution is sub-gaussian. However, benign overfitting in settings with input distributions heavier than sub-gaussian, that is, settings more robust to input variations, has not been extensively discussed.

\begin{figure}[ht]
    \centering
    \includegraphics[width=0.90\linewidth]{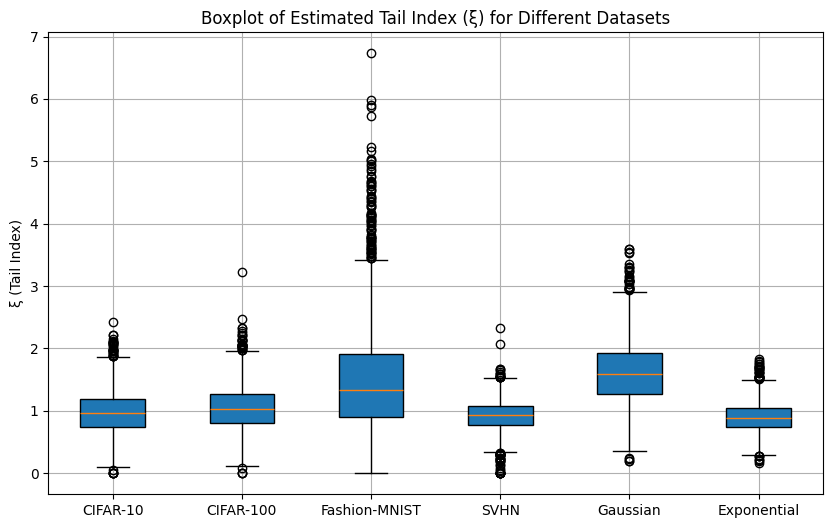}
    \caption{
        Boxplot of estimated tail index $\xi$ for feature vector components extracted from the intermediate layers of a CNN with ReLU activation, trained on various datasets (CIFAR-10 \cite{krizhevsky2009learning}, CIFAR-100 \cite{krizhevsky2009learning}, Fashion-MNIST \cite{xiao2017fashion}, SVHN \cite{netzer2011reading}). The tail index $\xi$ represents the heaviness of the distribution tails, with smaller values indicating heavier tails. The Gaussian and Exponential distributions are included for comparison purposes and were not passed through the CNN. The results indicate that the feature vectors for certain datasets, have heavier-tailed distributions than the Gaussian distribution. Further details are found in Appendix \ref{appendix:intro_xi}.
    }
    \label{fig:tail_index}
\end{figure}

\begin{figure}[ht]
    \centering
    \includegraphics[width=0.90\linewidth]{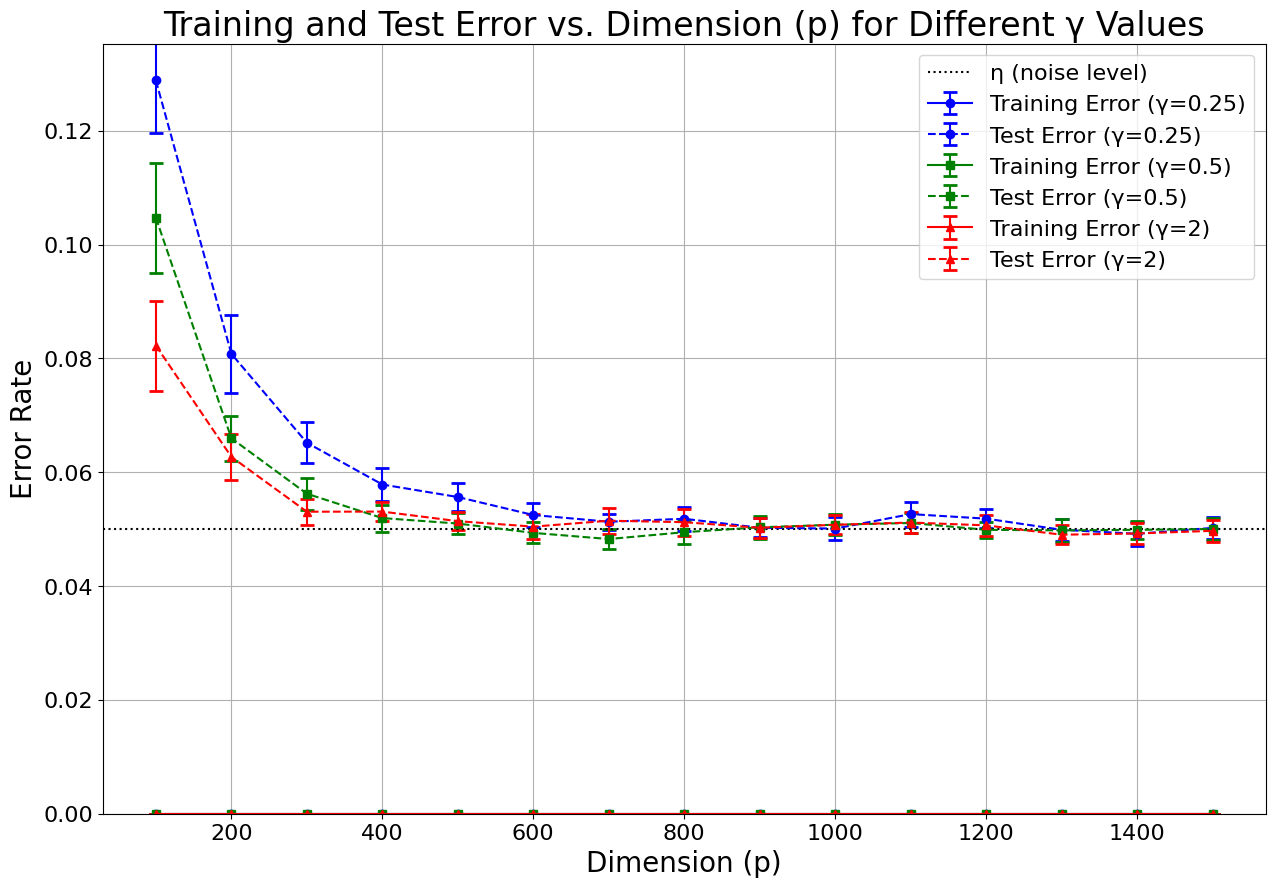}
    \caption{Training and test errors versus dimension $p$ for a maximum margin classifier. $n_\mathrm{train}=200$, $n_\mathrm{test}=1000$, $p$ ranges from 100 to 1500. Data is generated from a heavy-tailed setting using a generalized normal distribution, as detailed in Section \ref{section:heavy_tailed_setting} and Appendix \ref{appendix:intro_benign}. The shape parameters are $\gamma = 0.25, 0.5, 2$, with variance normalized to 1. Noise level $\eta$ is $0.05$ (dotted line). Solid and dashed lines show training and test errors, respectively, with $95$\% confidence intervals as error bars over 50 trials. Training error remains near zero, while test error stabilizes around the noise level as $p$ increases.} 
    \label{fig:gammma205025}
\end{figure}

Our numerical experiments indicate that feature vectors in convolutional neural networks (CNNs) with ReLU activation often exhibit distributions with tails heavier than sub-gaussian. The tail index $\xi$ can be intuitively understood as a parameter that characterizes the heaviness of the tails of a distribution. Specifically, for large values of $t$, the tail probability can be approximated as $\mathbb{P}[|X| > t] \simeq a \cdot \exp(-b \cdot t^\xi)$, where smaller values of $\xi$ correspond to heavier tails in the distribution. Figure \ref{fig:tail_index} indicates that these feature vectors have significantly heavy-tailed components. This finding emphasizes the necessity to extend benign overfitting analysis to accommodate a wider range of distributional settings. A more detailed explanation is provided in Appendix \ref{appendix:intro_xi}.

Moreover, our numerical experiments suggest that benign overfitting can occur even in mixture model settings where the distributions have heavier tails than the normal distribution, as seen in Figure \ref{fig:gammma205025}. Further details are available in Appendix \ref{appendix:intro_benign}. This motivates the exploration of benign overfitting in more general distributional frameworks. 

In this work, we focus on binary classification tasks where the input distribution is $\alpha$ sub-exponential with $\alpha \in (0, 2]$, implying tails heavier than sub-gaussian. We aim to establish generalization error bounds that demonstrate benign overfitting for a linear classifier trained using gradient descent on the unregularized logistic loss. Moreover, we derive an upper bound on the learning rate, a factor previously unexamined in this context, which plays a crucial role in demonstrating benign overfitting.

In this paper, we focus on the setting of Chatterji and Long \cite{JMLR:v22:20-974}, a pioneering work of benign overfitting theory on a simple model. Their results are extended to the heavy-tailed setting, and a more detailed discussion on the learning rate is provided.

\subsection{Related works}
\textbf{Benign overfitting in classification:}
Most related to our work is the theoretical analysis of benign overfitting in classification settings. This line of research aims to understand why classifiers that perfectly fit noisy training data can still generalize well to unseen data. Wang and Thrampoulidis \cite{wang2022binary} studied a setting where the two classes are symmetric mixture of Gaussian (or sub-gaussian) distributions, without label noise. Chatterji and Long \cite{JMLR:v22:20-974} studied overparameterized linear logistic regression on sub-gaussian mixture models with label flipping noise. They showed how gradient descent can train these models to achieve nearly optimal population risk. Cao et al. \cite{NEURIPS2021_46e0eae7} extended this work, tightening the upper bound in the case without label flipping noise and establishing a matching lower bound for overparameterized maximum margin interpolators. Wang et al. \cite{wang2021benign} extended the analysis of maximum margin classifiers to multiclass classification in overparameterized settings. Frei et al. \cite{frei2022benign} proved that a two-layer fully connected neural network exhibits benign overfitting under certain conditions, such as well-separated log-concave distribution and smooth activation function. Cao et al. \cite{cao2022benign} focused on the benign overfitting of two-layer convolutional neural networks.

\section{Preliminaries}
In this section, we introduce the definition of $\alpha$ sub-exponential random variables, the assumptions on the data generation process, and the maximum margin algorithm we consider. 

\subsection{Notation}
In this paper, we use the notation $[n]$ to denote the set $\{1,2,\ldots,n\}$ for a positive integer $n$. For a vector $x$, we use $\|x\|$ to denote its $\ell^2$ norm. For a matrix $A$, we use $\|A\|_{\mathrm{HS}}$ to denote its Hilbert–Schmidt norm and $\|A\|_{\mathrm{op}}$ to denote its operator norm. We use $s_k(A)$ to denote the $k$-th largest singular value of $A$. We use $O(\cdot)$ and $\Theta(\cdot)$ to refer to big-O and big-Theta notation.

\subsection{$\alpha$ sub-exponential random variable}\label{section:alpha_sub_exponential}
$\alpha$ sub-exponential random variables are random variables which have exponential type tails.

\begin{dfn}[$\alpha$ sub-exponential random variable, \cite{sambale2023some}]
    A random variable $X$ is called $\alpha$ sub-exponential if there is a positive constant $c_{\alpha}$ such that it holds
    \begin{align*}
        \mathbb{P}\left[|X-\mathbb{E}[X]|\geq t\right]\leq 2 \exp\left(-\frac{t^\alpha}{c_{\alpha}}\right)
    \end{align*}
    for any $t>0$. This is equivalent to having a finite exponential Orlicz norm:
    \begin{align*}
        \|X\|_{\psi_\alpha}:=\inf\left\{t>0 : \mathbb{E}\left[\exp\left(\frac{|X|^\alpha}{t^\alpha}\right)\right]\leq 2\right\} < \infty.
    \end{align*}
\end{dfn}

If $\alpha = 2$, we call the distribution sub-gaussian. If $\alpha = 1$, we call it sub-exponential. Here is an example of an $\alpha$ sub-exponential distribution:

\begin{ex}[Generalized normal distribution]
     The probability density function of the generalized normal distribution is defined as:
    \begin{align*}
        f(x; x_0, \sigma, \gamma) = \frac{\gamma}{2\sigma \Gamma(1/\gamma)} \exp\left(-\left|\frac{x-x_0}{\sigma}\right|^\gamma\right)
    \end{align*}
    where $\Gamma$ denotes the gamma function, $x_0$ is the location parameter, $\sigma > 0$ is the scale parameter, and $\gamma > 0$ is the shape parameter. Let $X$ be a random variable following the generalized normal distribution with location parameter $x_0 = 0$. Then $\gamma$ is the maximum value of $\alpha$ such that $\|X\|_{\psi_\alpha}$ is finite, and its exponential Orlicz norm is given by
    \begin{align*}
        \|X\|_{\psi_\gamma} = \frac{\sigma}{(1-2^{-\gamma})^{1/\gamma}}.
    \end{align*}
\end{ex}

\subsection{Data generation process}\label{section:heavy_tailed_setting}
We consider a heavy-tailed setting for binary classification, which is a relaxed setting of a standard mixture model setting (Chatterji and Long, 2021 \cite{JMLR:v22:20-974}; Frei et al., 2022 \cite{frei2022benign}). We first define a ``clean" distribution $\tilde{P}$ and then define the target distribution $P$ based on $\tilde{P}$: 

\begin{enumerate}
    \item Sample a ``clean" label $\tilde{y}\in\{\pm 1\}$ uniformly at random, $\tilde{y}\sim\mathrm{Uniform}(\{\pm 1\})$.
    \item Sample $q\sim P_{\mathrm{clust}}$ that satisfies: 
    \begin{itemize}
        \item $P_{\mathrm{clust}}:=P_{\mathrm{clust}}^{(1)}\times\cdots\times P_{\mathrm{clust}}^{(p)}$ is an arbitrary product distribution on $\mathbb{R}^p$ whose marginals are all mean-zero with the exponential Orlicz norm at most $1$, i.e., $\|X\|_{\psi_\alpha}\leq 1$ if $X \sim P_{\rm clust}^{(j)}$.
        \item For some $\kappa>0$, it holds that $$\mathbb{E}_{q\sim P_\mathrm{clust}}[\|q\|^2]\geq \kappa p.$$
    \end{itemize}
    \item For an arbitrary orthogonal matrix $U\in\mathbb{R}^{p\times p}$ and $\mu\in\mathbb{R}^p$, generate $\tilde{x} = Uq + \tilde{y}\mu$.
    \item Let $\tilde{P}$ be the distribution of $(\tilde{x},\tilde{y})$.
    \item For $\eta\in[0,1]$, let $P$ be an arbitrary distribution on $\mathbb{R}^p\times \{\pm1\}$ that satisfies:
    \begin{itemize}
        \item All marginal distributions of $P$ are the same as $\tilde{P}$.
        \item Total variation between $P$ and $\tilde{P}$ is at most $\eta$.
    \end{itemize}
\end{enumerate}

Let $ \mathcal{S}:= \{(x_1, y_1),\cdots,(x_n, y_n)\}$ be samples drawn according to $P$.

The reason we assume the $\alpha$ sub-exponential norm of each component is at most $1$ is only for simplifying the proofs and does not affect the main results of the paper, since rescaling the data does not affect the accuracy of the maximum margin algorithm.

This setting is a modification of Chatterji and Long's framework \cite{JMLR:v22:20-974}, where we have replaced the sub-gaussian norm with an exponential Orlicz norm. Moreover, it can be observed that when $\alpha=2$, this setting encompasses the original framework.

\subsection{Maximum margin algorithm}\label{section:maxmum_margin_algorithm}
We consider a linear classifier that takes the form $\mathrm{sign}(\theta\cdot x)$ trained by gradient descent as
\begin{align*}
    &\theta^{(t+1)} = \theta^{(t)} - \beta \nabla R(\theta^{(t)}) \\&\quad \mathrm{where} \quad R(\theta) :=\sum_{i=1}^n \log(1+\exp(-y_i \theta\cdot x_i)),
\end{align*}
where $\beta$ is the learning rate. In Soudry et al., 2018 \cite{soudry2018implicit}, they prove that if the dataset is linearly separable, in the large-$t$ limit, the normalized parameter of this classifier converges to the hard margin predictor:
\begin{gather*}
    \lim_{t\rightarrow\infty} \frac{\theta^{(t)}}{\|\theta^{(t)}\|} = \frac{w}{\|w\|}, \\
    w := \underset{u\in \mathbb{R}^p}{\mathrm{argmin}} \|u\| \\
    \text{such that } y_i(u \cdot x_i) \geq 1, \text{ for any } i \in [n].
\end{gather*}
They have proved this for a class of loss functions with certain smoothness.

\subsection{Assumptions}\label{section:assumption}
We assume that $\alpha$ and $\kappa$ are fixed constants. Let $X = [y_1 x_1 ,\cdots, y_n x_n]$, where $\{(x_k, y_k)\}_{k=1}^n$ are samples drawn according from the distribution $P$. We will prove the main theorem and corollaries under the following assumptions with a sufficiently large constant $C$ depending only on $\alpha$ and $\kappa$.
\begin{enumerate}[label=(A\arabic*)]
    \item The failure probability satisfies $\delta\in(0,\frac{1}{C})$,
    \item The number of samples satisfies $n\geq C\log\frac{1}{\delta}$,
    \item The dimension satisfies $$p\geq C \max\left(\|\mu\|^2 n, n^2 \left(\log\frac{n}{\delta}\right)^\frac{2}{\alpha}\right),$$
    \item The norm of the mean satisfies $\|\mu\|\geq C\left(\log\frac{n}{\delta}\right)^\frac{1}{\alpha}$,
    \item The learning rate satisfies \begin{align*}
        \beta \leq &\min \Bigg(8 (s_1(X))^{-2},\\& \frac{1}{c_2}\left(p + 2 n \left( \|\mu\|^2 + \sqrt{p} \left(\log\frac{n}{\delta}\right)^\frac{1}{\alpha}\right)\right)^{-1}\Bigg),
    \end{align*}
    where $c_2 = 2\max\left(\frac{8}{\kappa}, \frac{8}{\alpha}\Gamma\left(\frac{2}{\alpha}\right) + \kappa + 2 \right).$
\end{enumerate}

When $\alpha=2$, assumptions (A1)-(A4) correspond to those in Chatterji and Long \cite{JMLR:v22:20-974}.

Due to assumption (A1), if we require a lower failure probability, $C$ must be set large, which in turn requires tighter lower bounds for $n$, $p$ and $\|\mu\|$ in assumptions (A2)-(A4). Moreover, as the tail heaviness of the distribution grows (i.e., as $\alpha$ decreases), the lower bounds for $n$, $p$, and $\|\mu\|$ become tighter, and the upper bound on the learning rate $\beta$ becomes more restrictive.

\begin{ex}[Generalized noisy rare-weak model]
For any $\alpha\in(0,2]$, the model described above includes a special case called the generalized noisy rare-weak model, which is defined as follows:
\begin{itemize}
\item For any $j\in[p]$, $P_{\mathrm{clust}}^{(j)}$ is a generalized normal distribution with location
parameter $x_0 = 0$, shape parameter $\gamma = \alpha$, and scale parameter $\sigma$.
\item The mean vector $\mu\in\mathbb{R}^p$ has only $s$ non-zero components, all of which are equal to $\lambda > 0$, where $s$ and $\lambda$ are set appropriately to satisfy assumptions (A1)-(A4).
\end{itemize}
\end{ex}
If we require the exponential Orlicz norm to be less than 1, we need to adjust the scale parameter $\sigma$ of $P_{\text{clust}}^{(j)}$. Donoho and Jin \cite{donoho2008higher} studied this model where $\eta=0, \gamma = 1$ and $\sigma = 1$.

\section{Main results}

\subsection{Generalization bound}
We derive a generalization bound for the maximum
margin classifier in a relaxed standard mixture model.
\begin{thm}\label{mainthm}
    For any $\alpha \in (0,2]$ and $\kappa\in(0,1)$, there exists a constant $c>0$ such that, under assumptions (A1)-(A5), for all large enough $C$, with probability at least $1-\delta$, the maximum margin classifier $w$ satisfies
    \begin{equation*}
        \underset{(x,y)\sim P}{\mathbb{P}}\left[\mathrm{sign}(w\cdot x)\neq y\right] \leq \eta + \exp\left(-c\frac{\|\mu\|^{2\alpha}}{p^{\alpha/2}}\right).
    \end{equation*}
\end{thm}
The proof of this theorem is provided in Section \ref{section:sketch_of_proof}. This theorem reveals the relationship between the number of dimensions $p$ and $\|\mu\|$ in determining the success of learning. Specifically, when $\|\mu\|$ increases as $\|\mu\|=\Theta(p^\tau)$ for any $\tau\in(1/4,1/2]$, the misclassification error of the maximum margin classifier asymptotically approaches the noise level $\eta$.  The rate of increase in $\|\mu\|$ for benign overfitting is same as that proved by Chatterji and Long \cite{JMLR:v22:20-974} when $\alpha=2$. Therefore, our result shows that in high-dimensional feature spaces, if the signal is sufficiently strong, learning can be achieved while minimizing the impact of noise even for heavier tails ($\alpha<2$).

Here are the implications of Theorem \ref{mainthm} in the noisy rare-weak model where the mean vector $\mu$ has only $s$ non-zero elements and all non-zero elements equal $\gamma$.

\begin{cor}\label{cor:rare-weak}
    There exists a constant $c > 0$ such that, under assumptions (A1)-(A5), in the generalized noisy rare-weak model, for any $\lambda \geq 0$ and all large enough $C$, with probability $1 - \delta$, a maximum margin classifier $w$ satisfies
    \begin{equation*}
        \underset{(x,y)\sim P}{\mathbb{P}}\left[\mathrm{sign}(w\cdot x)\neq y\right] \leq \eta + \exp\left(-c\frac{(\lambda^2 s)^{\alpha}}{p^{\alpha/2}}\right).
    \end{equation*}
\end{cor}

We will consider $\lambda$ as fixed. Jin \cite{jin2009impossibility} demonstrated that for the noiseless rare-weak model, learning is impossible when $s = O(\sqrt{p})$ under the Gaussian assumption. Considering the fact that the Gaussian distribution is an $\alpha$ sub-gaussian for every $\alpha$ in $(0,2]$, their counterexample can show that our upper bound has optimality in a sense.
Strictly speaking, to fit Jin's model to our model, we need to adjust the scale parameter $\sigma$ of $P_{\text{clust}}^{(j)}$ to make the exponential Orlicz norm less than 1. However, this adjustment does not affect the accuracy of the maximum margin classifier.

\subsection{Learning rate}\label{section:learning_rate}
We perform a detailed analysis of sufficient conditions for the learning rate when benign overfitting occurs. To concretely calculate the assumption of the learning rate, a bound of the largest singular value of $X$ is used.

\begin{prop}[A bound of the singular values of $X$]\label{bounds_singular_value_2}
For any $\delta \geq 0$, with probability at least $1-\delta$, there are constants $c_5$, $c_6$, $c_7$, $c_8$ depending only on 
$\alpha$ such that
\begin{align*}
        &s_1(X) \\
        &\leq  \sqrt{p} \Bigg(c_5 +  \frac{c_6 \sqrt{n}}{p} \sum_{i=1}^p|\mu_i| + \frac{2n\|\mu\|^2}{p}  \\
        & \quad+\frac{c_7+c_8\max_{i}|\mu_i|
        \sqrt{n}}{p}\left(n \log 9 + \log\frac{4}{\delta}\right)^\frac{2}{\alpha}\Bigg).
    \end{align*}
\end{prop}
The proof of Proposition \ref{bounds_singular_value_2} is in Appendix \ref{proof_singlura_value}.
According to Proposition \ref{bounds_singular_value_2}, a sufficient condition for assumption (A5) can be expressed as assumption (A6):
\begin{enumerate}[label=(A\arabic*)]
    \setcounter{enumi}{5}
    \item The learning rate satisfies:
    \begin{align*}
        &\beta\leq \\& \min \Bigg(\frac{8}{p} \Bigg(c_5 +  \frac{c_6\sqrt{n}}{p} \sum_{i=1}^p|\mu_i| + \frac{2n\|\mu\|^2}{p}  \\
        & \quad+\frac{c_7+c_8\max_{i}|\mu_i|
        \sqrt{n}}{p}\left(n \log 9 + \log\frac{4}{\delta}\right)^\frac{2}{\alpha}\Bigg)^{-2}, \\
        &\quad \frac{1}{c_2 p}\left(1 + \frac{2n}{p} \left( \|\mu\|^2 + \sqrt{p} \left(\log\frac{n}{\delta}\right)^\frac{1}{\alpha}\right)\right)^{-1} \Bigg).
    \end{align*}
\end{enumerate}
By using assumption (A6) instead of (A5), we obtain Corollary \ref{A1A6}. 
\begin{cor}\label{A1A6}
    Under assumptions (A1)-(A4) and (A6) for all large enough $C$, with probability at least $1-2\delta$, the same generalization error bound as in Theorem \ref{mainthm} holds.
\end{cor}
Moreover, by Corollary \ref{A1A6}, we obtain Corollary \ref{order_of_beta_by_p} and \ref{order_of_beta_by_n}. The proofs of Corollaries \ref{order_of_beta_by_p} and \ref{order_of_beta_by_n} are in Appendix \ref{Proof_of_cor}.
\begin{cor}\label{order_of_beta_by_p}
    Under assumptions (A1)-(A4) for all large enough $C$, if $\beta$ satisfies 
    \begin{align*}
        \beta \leq c_{9} p^{-1}
    \end{align*}
    where $c_{9}$ is a constant depending on $\alpha, \kappa, \delta$, and $n$,
    with probability at least $1-2\delta$, the same generalization error bound as in Theorem \ref{mainthm} holds.
\end{cor}
This corollary implies that when \( p \) and \( \|\mu\| \) grow large while \( n \) and $\delta$ are fixed under assumptions (A3) and (A4), $\beta=O(p^{-1})$ is sufficient for the same result as Theorem \ref{mainthm}. The order remains unchanged even when $\alpha$ is small.
\begin{cor}\label{order_of_beta_by_n}
    Under assumptions (A1)-(A4) for all large enough $C$, if $\beta$ satisfies 
    \begin{align*}
         \beta \leq c_{10} p^{-1}\left(1 + n^{\frac{2}{\alpha}-1}(\log n)^{-\frac{1}{\alpha}}\right)^{-2}
    \end{align*}
    where $c_{10}$ is a constant depending on $\alpha, \kappa$, and $\delta$,
    with probability at least $1-2\delta$, the same generalization error bound as in Theorem \ref{mainthm} holds.
\end{cor}
Since $p \geq n^2$, it is straightforward to show that $\beta = O(p^{-\frac{2}{\alpha}})$ ensures the same generalization error bound as not only $p$ and $\|\mu\|$, but also $n$ grows under assumptions (A3) and (A4). As $\alpha$ decreases, the order also decreases, indicating that, for heavy-tailed distributions, the learning rate must be reduced accordingly.

\section{Sketch of proof of Theorem \ref{mainthm}}\label{section:sketch_of_proof}
In the lemmas of this section, we assume (A1)-(A5). The proofs of the lemmas in this section are provided in Appendix \ref{appendix:proof_of_lemma_in_main_theorem}. For simplicity, we assume $U=I$. This assumption can be made without loss of generality for the following reasons:
\begin{itemize}
    \item Transformation of the maximum margin classifier:\\
    If $w$ is the maximum margin classifier for the original data points $(x_1, y_1), \ldots, (x_n, y_n)$, then $Uw$ becomes the maximum margin classifier for the transformed data points $(Ux_1, y_1), \ldots, (Ux_n, y_n)$.
    \item Probability equivalence:\\
    The probability of misclassification remains unchanged whether we consider $y(w\cdot x) < 0$ or $y(Uw) \cdot (Ux) < 0$.
\end{itemize}
For the same reason as in section 4 of \cite{JMLR:v22:20-974}, without loss of generality, we can assume $\mathbb{P}(x=\tilde{x})=1$ and $\mathbb{P}(y\neq\tilde{y})=\eta$.

We define the sets of indices of ``noisy” and ``clean” samples.
\begin{dfn}
    Let $\mathcal{N}$ denote the set $\{k : y_k \neq \tilde{y}_k\}$ of indices of ``noisy”
    samples, and $\mathcal{C}$ denote the set $\{k : y_k = \tilde{y}_k\}$ indices of ``clean” samples.
\end{dfn}

Next, we define $z_k$, $\tilde{z}_k$, $\xi_k$, and $\tilde{\xi}_k$ to simplify the subsequent discussion.

\begin{dfn}
    For index $k \in [n]$ of each example, let $z_k$ denote $x_k y_k$ and let $\tilde{z}_k$ denote $\tilde{x}_k \tilde{y}_k$. Let $\xi_k$ denote $z_k - \mathbb{E}[z_k]$ and let $\tilde{\xi}_k$ denote $\tilde{z}_k - \mathbb{E}[\tilde{z}_k]$.
\end{dfn}

Then, $\xi_k$ and $\tilde{\xi}_k$ are $\alpha$ sub-exponential, and the following lemma holds:
\begin{lem}\label{xi}
    For any $k\in[n]$, 
    \begin{enumerate}
        \item $\mathbb{E}[z_k] = \mathbb{E}[\tilde{z}_k] = \mu$ and
        \item each component of $\xi_k$ and $\tilde{\xi}_k$ is $\alpha$ sub-exponential, with its exponential Orlicz norm at most $1$.
    \end{enumerate}
\end{lem}

The next lemma provides an upper bound for the misclassification error. This bound is expressed in terms of two factors:
\begin{itemize}
    \item The expected value of the margin on unperturbed data points, denoted as 
    \begin{align*}
        \underset{(\tilde{x},\tilde{y})\sim \tilde{P}}{\mathbb{E}}[\tilde{y}(w\cdot\tilde{x})],
    \end{align*}
    which equals $w\cdot\mu$.
    \item The Euclidean norm of the classifier vector $w$.
\end{itemize}

\begin{lem}\label{first_lemma}
    For any $w\in \mathbb{R}^p\setminus\{0\}$, there exists a positive constant $c$ such that
\begin{align*}
    \mathbb{P}_{(x,y)\sim P}\left[\mathrm{sign}(w\cdot x)\neq y\right] \leq \eta + 2\exp\left(-c \frac{|w\cdot\mu|^\alpha}{\|w\|^\alpha }\right).
\end{align*}
\end{lem}

The next lemma provides concentration arguments for $z_k$.

\begin{lem}\label{concentration_lemma}
    For any $\alpha\in(0,2]$ and $\kappa \in (0,1)$, there exists a constant $c_1\geq1$ such that, for any $c'$, for all large enough $C$, with probability at least $1-\delta$, the following holds:
    \begin{enumerate}
        \item For any $k \in [n]$, 
            \begin{align*}
                \frac{p}{c_1}\leq \|z_k\|^2\leq c_1 p.
            \end{align*}
        \item For any $i\neq j\in[n]$,
            \begin{align*}
                |z_i\cdot z_j|\leq c'\left(\|\mu\|^2 + \sqrt{p}\left(\log\frac{n}{\delta}\right)^\frac{1}{\alpha}\right).
            \end{align*}
        \item For any $k\in \mathcal{C}$,
            \begin{align*}
                |\mu\cdot z_k- \|\mu\|^2|<\frac{\|\mu\|^2}{2}.
            \end{align*}
        \item For any $k\in \mathcal{N}$,
            \begin{align*}
                |\mu\cdot z_k- (-\|\mu\|^2)|<\frac{\|\mu\|^2}{2}.
            \end{align*}
        \item The number of noisy samples satisfies $|\mathcal{N}|\leq(\eta+c')n$.
        \item The samples are linearly separable.
    \end{enumerate}
\end{lem}

From here on, we will assume that samples satisfy all the conditions of Lemma \ref{concentration_lemma}. 

The next lemma provides the bound on the ratio of losses when the loss function is the sigmoid loss.

\begin{lem}\label{loss_ratio_sigmoid}
    There exists a positive constant $c_3$ such that, for all large enough $C$, and any learning rate $\beta$ which satisfies
    \begin{equation*}
        \beta \leq \frac{1}{2c_1}\left( p + 2 n  \left( \|\mu\|^2 + \sqrt{p} \left(\log\frac{n}{\delta}\right)^\frac{1}{\alpha}\right)\right)^{-1},
    \end{equation*}
    for all iterations $t\geq 0$,
    \begin{equation*}
        \max_{i,j\in [n]}\left\{\frac {1+\exp(\theta^{(t)}\cdot z_j)}{1 + \exp(\theta^{(t)}\cdot z_i)}\right\} \leq c_3,
    \end{equation*}
    where $c_1$ is a constant which satisfies Lemma \ref{concentration_lemma}.
\end{lem}

Soudry et al. \cite{soudry2018implicit} provide results regarding the convergence behavior of $\theta^{(t)}$ when the data is linearly separable.

\begin{lem}[Soudry et al., 2018 \cite{soudry2018implicit}] \label{Soudry}
    For any linearly separable $\mathcal{S}$ and for $\beta \leq 8 (s_1(X))^{-2}$, we have
    \begin{equation*}
        \frac{w}{\|w\|}=\lim_{t\rightarrow\infty}\frac{\theta^{(t)}}{\|\theta^{(t)}\|}.
    \end{equation*}
\end{lem}

Using Lemmas \ref{concentration_lemma}, \ref{loss_ratio_sigmoid}, and \ref{Soudry}, we derive Lemma \ref{wmu}.

\begin{lem}\label{wmu}
    For any $\kappa\in(0,1)$, there exists a positive constant $c_4$ such that, for any large enough $C$, with probability at least $1-\delta$, the maximum margin weight vector $w$ satisfies,
    \begin{equation*}
        \mu \cdot w\geq \frac{\|w\|\|\mu\|^2}{c_4\sqrt{p}}.
    \end{equation*}
\end{lem}

By Lemmas \ref{first_lemma} and \ref{wmu}, we have Theorem \ref{mainthm}.

\section{Simulation}\label{section:simulation}
We conducted simulation studies to assess the performance of the maximum margin classifier across various conditions, specifically focusing on how dimensionality, tail heaviness, and learning rates interact. The simulation was designed with the following parameters: the training set consisted of $500$ samples, we used $1000$ test samples to assess generalization, and each experiment was repeated $5$ times, and the results were averaged to ensure robustness. The link to the detailed code for the experiments is provided in Appendix \ref{appendix:infrastructure}.

\begin{figure}[ht]
    \centering
    \includegraphics[width=0.80\linewidth]{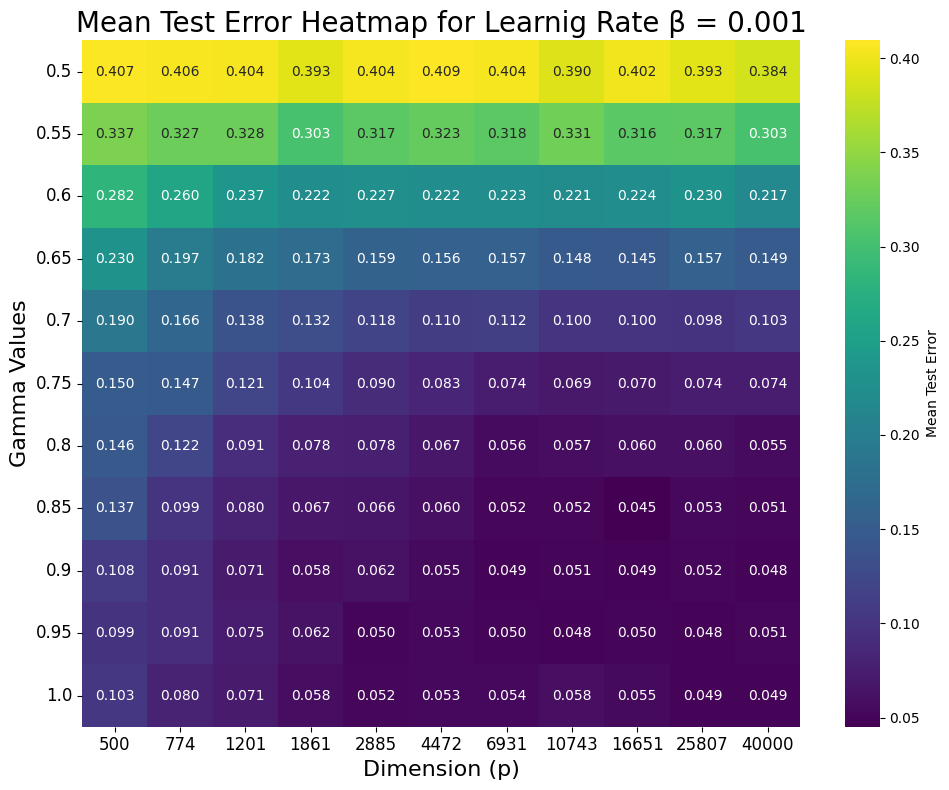}
    \caption{A heatmap showing the mean test error for $\beta=0.001$  with the horizontal axis representing the dimension $p$ and the vertical axis representing the shape parameter $\gamma$.}
    \label{figure_heatmap1}
\end{figure}
\begin{figure}[ht]
    \centering
    \includegraphics[width=0.80\linewidth]{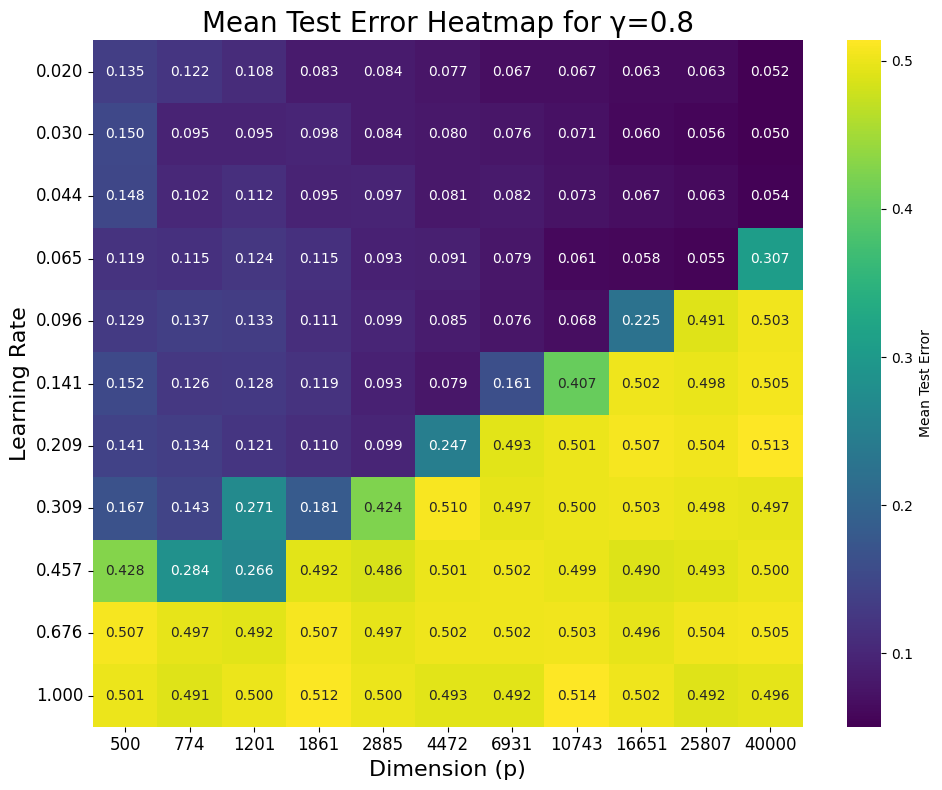}
    \caption{A heatmap showing the mean test error for $\gamma=0.8$  with the horizontal axis representing the dimension $p$ and the vertical axis representing the learning rate $\beta$.}
    \label{figure_heatmap2}
\end{figure}
\begin{figure}[ht]
    \centering
    \includegraphics[width=0.80\linewidth]{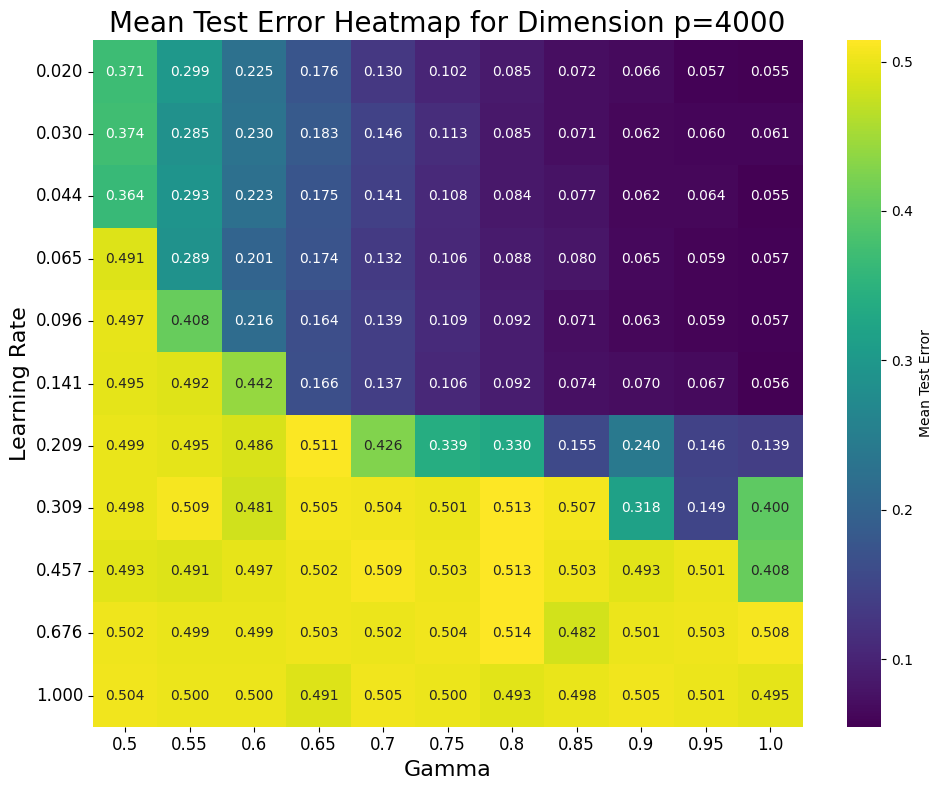}
    \caption{A heatmap showing the mean test error for $p=4000$ with the horizontal axis representing the shape parameter $\gamma$ and the vertical axis representing the learning rate $\beta$.}
    \label{figure_heatmap3}
\end{figure}

\subsection{Data generation}
The data was generated by the heavy-tailed setting, as described in Section \ref{section:heavy_tailed_setting}. We set $P_\mathrm{clust}$ to be the generalized normal distribution with a scale parameter of $1$ and a shape parameter \(\gamma\) ranging from $0.5$ to $1$. This distribution is used to control the tail behavior of the data, where smaller values of \(\gamma\) correspond to heavier tails.

For the mean vector, \(\mu\), the first \( \lfloor p^{2/3}\rfloor  \) elements were set to 1, and the remaining elements were set to 0, ensuring that \( \|\mu\| = \Theta(p^{1/3}) \). We chose an orthogonal matrix \(U\) such that \(U = I\), the identity matrix.

We also incorporated label noise by flipping the labels with a noise level of \(\eta = 0.05\), meaning that each true label was flipped with a probability of \( \eta \).

\subsection{Model training}
We used the maximum margin classifier, as described in Section \ref{section:maxmum_margin_algorithm}. The model was trained for $100000$ epochs to ensure convergence. We conducted three different numerical experiments to observe how these conditions influence test error.

\subsection{Results and discussion}

\paragraph{Experiment 1: interaction between dimension $p$ and shape parameter $\gamma$ (Figure \ref{figure_heatmap1})}
In the first experiment, we investigated how the interaction between the dimensionality $p$ and the shape parameter $\gamma$ influences model performance. As $p$ increases, the test error initially decreases and then stabilizes around the noise level. For smaller values of $\gamma$ (heavier tails), the stabilization occurs more slowly, indicating that learning from heavy-tailed distributions is more challenging. In contrast, for larger values of $\gamma$ (lighter tails), the model converges faster.

These results suggest that high-dimensional parameter spaces permit benign overfitting, regardless of the tail heaviness. However, for heavier-tailed distributions (smaller $\gamma$), more dimensions are required to achieve similar performance compared to lighter-tailed distributions. This is consistent with the theoretical assumptions (A3) and (A4).

\paragraph{Experiment 2: impact of dimension $p$ and learning rate $\beta$ (Figure \ref{figure_heatmap2})}
In the second experiment, we explored how the interaction between dimensionality $p$ and learning rate $\beta$ affects the test error. For large $p$, benign overfitting does not occur unless a small learning rate $\beta$ is chosen. If $\beta$ is too large, the learning process struggles to make progress.

As $p$ increases, the generalization error should decrease, as predicted by the benign overfitting bound. However, if the learning rate $\beta$ is not sufficiently small, the conditions outlined in Corollary 7 are not satisfied, and the learning process does not perform well. Our simulations suggest that in high-dimensional parameter spaces, the learning rate $\beta$ must be reduced to enable benign overfitting.

\paragraph{Experiment 3: impact of shape parameter $\gamma$ and learning rate $\beta$ (Figure \ref{figure_heatmap3})}
In the third experiment, we fixed the dimensionality at $p = 4000$ and examined the interaction between the shape parameter $\gamma$ and the learning rate $\beta$. For smaller $\gamma$ (heavier tails), the model is more sensitive to the choice of $\beta$. In particular, larger values of $\beta$ result in higher test errors for smaller $\gamma$. Conversely, for larger values of $\gamma$ (lighter tails), the model performs well even with larger learning rates. This suggests that careful tuning of the learning rate is crucial when dealing with heavy-tailed distributions to achieve benign overfitting.

These findings align with the theoretical results, indicating that when $\gamma$ is small, if $\beta$ is not sufficiently small, the condition on $\beta$ specified in Corollary \ref{order_of_beta_by_n} is violated.

\section{Conclusion}
Our research extends the analysis of benign overfitting in binary classification problems to heavy-tailed input distributions, specifically $\alpha$ sub-exponential distributions where $\alpha \in (0,2]$. The main findings of this study are:

\begin{enumerate}
    \item We derived generalization bounds for maximum margin classifiers in this heavy-tailed setting, showing that benign overfitting can occur under certain conditions on dimensionality $p$ and the feature vector magnitude $\|\mu\|$.
    
    \item Our results demonstrate that as the number of dimensions $p$ increases and the feature vector magnitude $\|\mu\|$ scales as $\Theta(p^d)$ for $d \in (1/4,1/2]$, the misclassification error approaches the noise level $\eta$ even under the heavy-tailed setting.
    
    \item In the context of the noisy rare-weak model, our upper bounds suggest that the maximum margin classifier can succeed arbitrarily close to the known impossibility threshold of $s = O(\sqrt{p})$.

    \item We showed that the upper bound on the learning rate for benign overfitting, and demonstrated that when \(n\) is fixed, the bound is of order \(p^{-1}\), while in the case where \(n, p, \|\mu\|\) are large, we observed that the upper bound decreases as \(\alpha\) increases.

    \item By conducting simulations, we confirmed that the relationship between the number of parameters, the tail heaviness, and the learning rate when benign overfitting occurs follows the same trend as that derived theoretically.
\end{enumerate}

These findings significantly contribute to our understanding of benign overfitting by showing that the phenomenon is not limited to sub-gaussian distributions but extends to heavier-tailed inputs as well. This research bridges a gap between theory and practice, as real-world data often exhibit heavier tails than the Gaussian distribution.

Our work opens up several avenues for future research:

\begin{enumerate}
    \item Investigation of benign overfitting in even heavier-tailed distributions, such as those with polynomial tails.
    
    \item Extension of the analysis to multi-class classification problems with heavy-tailed inputs.
    
    \item Exploration of the implications of these findings for deep learning models, which often deal with high-dimensional, heavy-tailed data.
    
    \item Development of new learning algorithms that explicitly leverage the properties of heavy-tailed distributions to achieve better generalization in high-dimensional settings.
\end{enumerate}

In conclusion, this study provides a significant step towards understanding the phenomenon of benign overfitting in more realistic data settings. By extending the theory to heavy-tailed distributions, we have broadened the applicability of benign overfitting results to a wider range of practical scenarios, potentially impacting the design and analysis of machine learning algorithms for complex, real-world data.

\section*{Acknowledgements}
We would like to thank our colleagues for their valuable feedback and suggestions that helped improve this work. This research was supported in part by RIKEN AIP and JSPS KAKENHI (JP22K03439).

\bibliography{references}

\onecolumn

\appendix
\aistatstitle{Benign Overfitting under Learning Rate Conditions for \\$\alpha$ Sub-exponential Inputs: \\
Supplementary Materials}

\section{Concentration inequality}\label{Appendix}
    In this section, we introduce the concentration inequalities for $\alpha$ sub-exponential random variables. In our proof, we apply the following two concentration inequalities.
    
    \begin{prop}[A special case of Theorem 1.5 in \cite{gotze2021concentration}]\label{sum_concentration_inequality}
        Let $\alpha\in(0,2]$ and $K$ be a positive constant and $a\in\mathbb{R}^n$ be a constant vector. Let $X_1, \ldots, X_n$ be independent mean-zero random variables satisfying $\|X_i\|_{\psi_\alpha}\leq K$. Then, there exists a positive constant $c_\alpha$ such that for any $t > 0$ it holds
        \begin{equation*}
            \mathbb{P}\left[\left|\sum_{i=1}^n a_i X_i\right|\geq t\right]\leq 2 \exp\left(-\frac{1}{c_\alpha}\frac{t^\alpha}{K^\alpha\|a\|^\alpha}\right).
        \end{equation*}
    \end{prop}
    
    This is a special case of Theorem 1.5 of G\"{o}tze et al. \cite{gotze2021concentration}.
    
    \begin{thm}[Extended Hanson-Wright inequality \cite{sambale2023some}]\label{Extended_Hanson-Wright_inequality}
        Let $\alpha\in(0,2]$ and $K$ be a positive constant. Let $X_1,\ldots,X_n$ be independent mean-zero random variables such that $\|X_i\|_{\psi_\alpha}\leq K$, the corresponding random vector $X$ be $(X_1,\ldots,X_n)^T$ and $A\in \mathbb{R}^{n\times n}$ be a symmetric matrix. Then, there exists a positive constant $c_\alpha$ such that it holds
        \begin{equation*}
            \mathbb{P}[|X^T A X -\mathbb{E}[X^T A X]|\geq t]\leq 2\exp \left(-\frac{1}{c_\alpha}\min\left(\frac{t^2}{K^4\|A\|_{\mathrm{HS}}^2},\left(\frac{t}{K^2\|A\|_{\mathrm{op}}}\right)^\frac{\alpha}{2}\right)\right)
        \end{equation*}
        for any $t\geq 0$.
    \end{thm}
\section{Missing proofs}\label{appendix:missing_proofs}
    \subsection{Proofs of lemmas used in the proof of Theorem \ref{mainthm}}\label{appendix:proof_of_lemma_in_main_theorem}
    \subsubsection{Proof of Lemma \ref{xi}}
        \begin{proof}[Proof of Lemma \ref{xi}]
            By definition of $\tilde{z}_k$ and $z_k$,
            \begin{gather*}
                \mathbb{E}[\tilde{z}_k] =\mathbb{E}[\tilde{x}_k\tilde{y}_k]= \mathbb{E}[( q + \tilde{y}_k\mu)\tilde{y}_k] = \mu,\\
                \mathbb{E}[z_k] = \mathbb{E}[\tilde{x}_k y_k]= \mathbb{E}[( q + y_k\mu)y_k]  = \mu.
            \end{gather*}
            Together with the definition of $\xi_k$,
            \begin{gather*}
                \|\xi_{kl}\|_{\psi_\alpha} = \left\|q_l y_k\right\|_{\psi_\alpha} = \left\|q_l \right\|_{\psi_\alpha} \leq 1,\\
                 \|\tilde{\xi}_{kl}\|_{\psi_\alpha} = \left\|q_l \tilde{y}_k\right\|_{\psi_\alpha} = \left\|q_l \right\|_{\psi_\alpha} \leq 1.
            \end{gather*}
        \end{proof}
    \subsubsection{Proof of Lemma \ref{first_lemma}}
   
        \begin{proof}[Proof of Lemma \ref{first_lemma}]
            Following the proof of Chatterji and Long \cite{JMLR:v22:20-974}, we have
            \begin{align*}
                \mathbb{P}_{(x,y)\sim P}\left[\mathrm{sign}(w\cdot x)\neq y\right] &= \mathbb{P}_{(x,y)\sim P}\left[y(w\cdot x)<0\right]\\
                & \leq \eta + \mathbb{P}_{(\tilde{x},\tilde{y})\sim \tilde{P}}\left[\tilde{y}(w\cdot \tilde{x})<0\right]\\
                &= \eta + \mathbb{P}_{(\tilde{x},\tilde{y})\sim \tilde{P}}\left[\left(\frac{w}{\|w\|}\cdot \tilde{y}\tilde{x}\right)<0\right]\\
                &= \eta + \mathbb{P}_{(\tilde{x},\tilde{y})\sim \tilde{P}}\left[\left(\frac{w}{\|w\|}\cdot \tilde{\xi}\right)< - \frac{w}{\|w\|}\cdot\mu\right].
            \end{align*}
            Applying Proposition \ref{sum_concentration_inequality}, there exists a positive constant $c$ such that
            \begin{align*}
                \mathbb{P}_{(\tilde{x},\tilde{y})\sim \tilde{P}}\left[\left(\frac{w}{\|w\|}\cdot \tilde{\xi}\right)< - \frac{w}{\|w\|}\cdot\mu\right] 
                &\leq 
                 2\exp\left(-c\frac{|w\cdot\mu|^\alpha}{\|w\|^\alpha }\right),
            \end{align*}
            which completes our proof.
        \end{proof}
    \subsubsection{Proof of Lemma \ref{concentration_lemma}}\label{Proof_of Lemma_concentration_lemma}
        In this section, we prove Lemma \ref{concentration_lemma} by using concentration inequalities from Section \ref{Appendix}. We assume assumptions (A1)-(A4) hold, and decompose Lemma \ref{concentration_lemma} into six different parts. We prove that each separate lemma holds with probability at least $1-\delta/6$.

        \begin{lem}\label{lem2-1}
            For any $\alpha\in(0,2]$ and $\kappa \in (0,1)$, there exists a constant $c\geq1$ such that, for all large enough $C$, with probability at least $1-\delta/6$, for any $k\in[n]$,
            \begin{equation*}
                \frac{p}{c} \leq \|z_k\|^2 \leq cp.
            \end{equation*}
        \end{lem}

        \begin{proof}
            By Theorem \ref{Extended_Hanson-Wright_inequality} with $A=I$, there exists a positive constant $c$ such that
            \begin{equation*}
                \mathbb{P}\left[\left|\|\xi_k\|^2 -\mathbb{E}[\|\xi_k\|^2]\right|\geq t\right]\leq 2 \exp \left(-\frac{1}{c} \min \left(\frac{t^2}{p},t^\frac{\alpha}{2}\right)\right).
            \end{equation*}
            By setting $t=\frac{\kappa p}{2}$
            \begin{equation*}
                2 \exp \left(-\frac{1}{c} \min \left(\frac{t^2}{p},t^\frac{\alpha}{2}\right)\right) = 2 \exp \left(-\frac{1}{c} \min \left(\left(\frac{\kappa}{2}\right)^2 p,\left(\frac{\kappa}{2}\right)^\frac{\alpha}{2} p^\frac{\alpha}{2}\right)\right).
            \end{equation*}
            By assumption (A3), we have $p\geq C \left(\log\frac{n}{\delta}\right)^\frac{2}{\alpha}$. There exists a large enough constant $C$ such that
            \begin{equation*}
                2 \exp \left(-\frac{1}{c} \min \left(\left(\frac{\kappa}{2}\right)^2 p,\left(\frac{\kappa}{2}\right)^\frac{\alpha}{2} p^\frac{\alpha}{2}\right)\right)\leq \frac{\delta}{6n}.
            \end{equation*}
            Thus,
            \begin{equation}\label{delta/6}
                \mathbb{P}\left[\left|\|\xi_k\|^2 - \mathbb{E}[\|\xi_k\|^2]\right|\geq \frac{\kappa p}{2}\right]\leq \frac{\delta}{6n}.
            \end{equation}
            Recalling the assumption $\mathbb{E}[\|q\|^2]\geq \kappa p$, we have
            \begin{equation*}
                \mathbb{E}[\|\xi_k\|^2] = \mathbb{E}\left[\|z_k -\mathbb{E}[z]\|^2\right] = \mathbb{E}[\|q\|^2] \geq \kappa p.
            \end{equation*}
            Let $\{\xi_{kj}\}_{j=1}^p$ be elements of $\xi_k$. By $\|\xi_{kj}\|_{\psi_\alpha}\leq 1$ for each $j$, 
            \begin{align*}
                \mathbb{E}[|\xi_{kj}|^2] &= 2\int_0^\infty t \mathbb{P}[|\xi_{kj}|\geq t]dt \\
                &\leq 2\int_0^\infty t \cdot 2\exp(-t^\alpha)dt \\
                &= 4\int_0^\infty t \exp(-t^\alpha)dt \\
                &= \frac{4}{\alpha}\int_0^\infty u^{2/\alpha - 1} \exp(-u) du \\
                &= \frac{4}{\alpha}\Gamma\left(\frac{2}{\alpha}\right).
            \end{align*}
            Thus, $\mathbb{E}[\|\xi_k\|^2]\leq \frac{4p}{\alpha}\Gamma\left(\frac{2}{\alpha}\right)$. Because of this and (\ref{delta/6}), with probability at least $1-\frac{\delta}{6n}$,
            \begin{equation*}
                \frac{\kappa p}{2}\leq \|\xi_k\|^2 \leq \left(\frac{4}{\alpha}\Gamma\left(\frac{2}{\alpha}\right) + \frac{\kappa}{2} \right) p.
            \end{equation*}
            
            Suppose $k\in \mathcal{C}$, and let $\xi_k = z_k - \mu$. Then, the following inequalities hold.
            \begin{enumerate}
                \item $\|z_k - \mu\|^2 \leq 2\|z_k\|^2 + 2\|\mu\|^2$.
                \item $\|\mu\|^2 < \frac{p}{C}$ by assumption (A3).
                \item $\|\xi_k\|^2 = \|z_k - \mu\|^2 \geq \frac{\kappa p}{2}$ with probability at least $1-\frac{\delta}{6}$.
            \end{enumerate}
            Combining these inequalities, we obtain
            \begin{align*}
            \|z_k\|^2 &\geq \frac{1}{2}\|z_k - \mu\|^2 - \|\mu\|^2 \label{eq:ineq1} \\
            &\geq \frac{1}{2}\left(\frac{\kappa p}{2}\right) - \left(\frac{p}{C}\right)\\
            &= \frac{\kappa p}{4} - \frac{p}{C}.
            \end{align*}  
            For sufficiently large $C$, we can ensure that $\frac{p}{C} < \frac{\kappa p}{8}$. Thus,
            \begin{equation}
            \|z_k\|^2 > \frac{\kappa p}{4} - \frac{\kappa p}{8} = \frac{\kappa p}{8}.
            \end{equation}
            Therefore, with probability at least $1-\frac{\delta}{6n}$, we have $\|z_k\|^2 > \frac{\kappa p}{8}$ for sufficiently large $C$.
            Again by $\|z_k\|^2 \leq 2\|z_k -\mu\|^2 + 2\|\mu\|^2$,
            \begin{align*}
                \|z_k\|^2 &\leq 2\|z_k -\mu\|^2 + 2\|\mu\|^2 \\
                &\leq 2\left(\frac{4}{\alpha}\Gamma\left(\frac{2}{\alpha}\right) + \frac{\kappa}{2}\right)p + 2\|\mu\|^2 \\
                &< 2\left(\frac{4}{\alpha}\Gamma\left(\frac{2}{\alpha}\right) + \frac{\kappa}{2}\right)p + \frac{2p}{C} \\
                &< \left(\frac{8}{\alpha}\Gamma\left(\frac{2}{\alpha}\right) + \kappa + 2\right)p.
            \end{align*}
            Therefore, setting $c=\max\left(\frac{8}{\kappa}, \frac{8}{\alpha}\Gamma\left(\frac{2}{\alpha}\right) + \kappa + 2 \right)$, we have
            \begin{equation*}
                \frac{p}{c} \leq \|z_k\|^2 \leq cp
            \end{equation*}
            for any $k\in\mathcal{C}$.
            A similar argument holds for $k\in\mathcal{N}$.
        \end{proof}
        
        \begin{lem}
            There exists $c\geq 1$ such that, for any large enough $C$, with probability at least $1-\frac{\delta}{6}$, for any $i\neq j \in[n]$,
            \begin{equation*}
                |z_i\cdot z_j|\leq c\left(\|\mu\|^2 + \sqrt{p}\left(\log\frac{n}{\delta}\right)^\frac{1}{\alpha}\right).
            \end{equation*}
        \end{lem}
        
        \begin{proof}
            Applying Theorem \ref{Extended_Hanson-Wright_inequality} as in the proof of Lemma \ref{lem2-1} and using the union bound method, we have
            \begin{equation*}
                \mathbb{P}[\exists i \in[n], \|\xi_i\|\geq \sqrt{p}]\leq \frac{\delta}{24}.
            \end{equation*}
            For any pair $i,j\in[n]$, we have
            \begin{equation}\label{delta/24}
                \mathbb{P}[|\xi_i\cdot\xi_j|\geq t]\leq \mathbb{P}[|\xi_i\cdot\xi_j|\geq t \mid \|\xi_j\|\leq\sqrt{p}]+ \mathbb{P}[\|\xi_j\|>\sqrt{p}].
            \end{equation}
            Regarding $\xi_j$ as fixed, by Proposition \ref{sum_concentration_inequality} there exists a positive constant $c_2$ such that
            \begin{equation*}
                \mathbb{P}[|\xi_i\cdot\xi_j|\geq t] = \mathbb{P}\left[\frac{\xi_j}{\|\xi_j\|}\cdot\xi_i\geq\frac{t}{\|\xi_j\|}\right]\leq 2 \exp\left(-c_2\frac{t^\alpha}{\|\xi_j\|^\alpha}\right).
            \end{equation*}
            Therefore,
            \begin{align*}
                \mathbb{P}[|\xi_i\cdot\xi_j|\geq t \mid \|\xi_j\| \leq \sqrt{p}] &\leq 2 \exp \left(-c_2\frac{t^\alpha}{p^\frac{\alpha}{2}}\right),\\
                \mathbb{P}[|\xi_i\cdot\xi_j|\geq t ] &\leq 2 \exp \left(-c_2\frac{t^\alpha}{p^\frac{\alpha}{2}}\right) + \mathbb{P}[\|\xi_j\|>\sqrt{p}].
            \end{align*}
            By the union bound method,
            \begin{equation*}
                \mathbb{P}\left[\exists i\neq j \in [n],|\xi_i\cdot\xi_j|\geq t \right]\leq 2 n^2 \exp \left(-c_2\frac{t^\alpha}{p^\frac{\alpha}{2}}\right) + \mathbb{P}[\exists j\in[n], \|\xi_j\|>\sqrt{p}].
            \end{equation*}
            Setting $t=c_3\left(p^\frac{\alpha}{2} \log\frac{n}{\delta}\right)^\frac{1}{\alpha}$, for a large enough $c_3$, we have
            \begin{equation*}
                \mathbb{P}\left[\exists i\neq j\in[n], |\xi_i\cdot\xi_j|\geq c_3\left(p^\frac{\alpha}{2} \log\frac{n}{\delta}\right)^\frac{1}{\alpha}\right] \leq \frac{\delta}{24} + \mathbb{P}[\exists j\in[n], \|\xi_j\|>\sqrt{p}].
            \end{equation*}
            Together with the inequality (\ref{delta/24}), we have,
            \begin{equation}\label{1_delta/12}
                \mathbb{P}\left[\exists i\neq j\in[n], |\xi_i\cdot\xi_j|\geq c_3\left(p^\frac{\alpha}{2} \log\frac{n}{\delta}\right)^\frac{1}{\alpha}\right] \leq \frac{\delta}{12}.
            \end{equation}
            By Proposition \ref{sum_concentration_inequality}, there exists a constant $c_4$ such that
            \begin{align*}
                \mathbb{P}[|\mu\cdot z_k| > \|\mu\|^2] &= \mathbb{P}\left[\left|\frac{\mu}{\|\mu\|}\cdot z_k\right|> \|\mu\|\right]\leq 2\exp \left(-c_4\|\mu\|^\alpha\right).
            \end{align*}
            By assumption (A4), for large enough $C$ we have
            \begin{equation*}
                \mathbb{P}[|\mu\cdot z_k| > \|\mu\|^2] \leq \frac{\delta}{12n}.
            \end{equation*}
            By taking a union bound, we have
            \begin{equation}\label{2_delta/12}
                \mathbb{P}[\exists k ,|\mu\cdot z_k| > \|\mu\|^2] \leq \frac{\delta}{12}.
            \end{equation}
            Due to inequalities (\ref{1_delta/12}) and (\ref{2_delta/12}), with probability at least $1-\frac{\delta}{6}$,
            \begin{align*}
                |z_i\cdot z_j| &= \Big|(z_i - \mathbb{E}[z_i])\cdot (z_j -\mathbb{E}[z_j]) - \mathbb{E}[z_i]\cdot \mathbb{E}[z_j] + \mathbb{E}[z_i]\cdot z_j +\mathbb{E}[z_j]\cdot z_i\Big|\\
                & = \left|\xi_i \cdot \xi_j -\|\mu\|^2 + \mu\cdot z_j + \mu \cdot z_i\right|\\
                &\leq |\xi_i \cdot \xi_j| +\|\mu\|^2 + |\mu\cdot z_j| + |\mu \cdot z_i|\\
                &\leq 3\|\mu\|^2 + c\left(p^\frac{\alpha}{2} \log\frac{n}{\delta}\right)^\frac{1}{\alpha}.
            \end{align*}
        \end{proof}
        
        \begin{lem}\label{lem2-3}
            For any large enough $C$, with probability at least $1-\frac{\delta}{6}$, for $k\in\mathcal{C}$,
            \begin{equation*}
                |\mu\cdot z_k - \|\mu\|^2| \leq \frac{\|\mu\|^2}{2}.
            \end{equation*}
        \end{lem}
        
        \begin{lem}\label{lem2-4}
            For any large enough $C$, with probability at least $1-\frac{\delta}{6}$, for $k\in\mathcal{N}$,
            \begin{equation*}
                |\mu\cdot z_k - (-\|\mu\|^2)| \leq \frac{\|\mu\|^2}{2}.
            \end{equation*}
        \end{lem}
        
        Lemma \ref{lem2-3} and Lemma \ref{lem2-4} can be proven by the same logic. We will prove Lemma \ref{lem2-3}.
        
        \begin{proof}
            If $k \in \mathcal{C}$, 
            \begin{equation*}
                \mu\cdot z_k - \|\mu\|^2 = \mu\cdot \xi_k.
            \end{equation*}
            Since the exponential Orlicz norm of $\xi_k$ is at most $1$, by Proposition \ref{sum_concentration_inequality}, there exists a positive constant $c$ such that
            \begin{align*}
                \mathbb{P}\left[\left|\mu\cdot z_k - \|\mu\|^2\right|\geq \frac{\|\mu\|^2}{2}\right] &= \mathbb{P}\left[\left|\frac{\mu}{\|\mu\|}\cdot\xi_k\right|\geq \frac{\|\mu\|}{2}\right]\\
                &\leq 2 \exp \left(-c\frac{\|\mu\|^\alpha}{2^\alpha}\right).
            \end{align*}
            By assumption (A4), for large enough constant $C$ we have
            \begin{equation*}
                \mathbb{P}\left[\left|\mu \cdot z_k - \|\mu\|^2\right|\geq \frac{\|\mu\|^2}{2} \right]\leq \frac{\delta}{6n}.
            \end{equation*}
            Taking a union bound, we have
            \begin{equation*}
                \mathbb{P}\left[\exists k \in \mathcal{C}, \left|\mu\cdot z_k - \|\mu\|^2\right|\geq \frac{\|\mu\|^2}{2}\right] \leq \frac{\delta}{6},
            \end{equation*}
            which completes our proof.
        \end{proof}
        
        \begin{lem}
            For any $c'>0$, for any large enough $C$, with probability at least $1-\frac{\delta}{6}$, the number of noisy samples satisfies $|\mathcal{N}|\leq (\eta+c')n$.
        \end{lem}
        
        \begin{proof}
            \begin{align*}
                \mathbb{E}\left[|\mathcal{N}| \right] = \sum_{k=1}^n\mathbb{E}\left[1_{\{y_k \neq \tilde{y}_k\}} \right] = \sum_{k=1}^n \mathbb{P}[y_k \neq \tilde{y}_k]  = n\eta.
            \end{align*}
            By Hoeffding's inequality, 
            \begin{align*}
                \mathbb{P}\left[|\mathcal{N}|\geq (\eta+c')n\right] &= \mathbb{P}\left[\frac{1}{n}\sum_{k=1}^n\left(1_{\{y_k \neq \tilde{y}_k\}}-\mathbb{E}[1_{\{y_k \neq \tilde{y}_k\}}]\right)\geq c'\right] \\
                &\leq 2 \exp\left(-2{c'}^2n\right) \\
                &\leq \frac{\delta}{6}.
            \end{align*}
            The last inequality holds due to assumption (A2).
        \end{proof}
        
        \begin{lem}
            If the following conditions hold, for any large enough $C$, $\{(x_k,y_k)\}_{k=1}^n$ are linearly separable.
            \begin{enumerate}
                \item There exists a positive constant $c$ such that for any $k\in[n]$ 
                \begin{equation*}
                    \frac{p}{c}\leq\|z_k\|^2\leq cp.
                \end{equation*}
                \item There exists a positive constant $c$ such that for any $i\neq j\in[n]$  
                \begin{equation*}
                    |z_i\cdot z_j|\leq c\left(\|\mu\|^2 + \sqrt{p}\left(\log\frac{n}{\delta}\right)^\frac{1}{\alpha}\right).
                \end{equation*}
            \end{enumerate}
        \end{lem}
        
        \begin{proof}
            Let $v$ be $\sum_{i\in[n]} z_i$. For each $k\in[n]$ and any $\delta>0$,
            \begin{align*}
                y_k v\cdot x_k &= \sum_{i\in[n]} z_i\cdot z_k \\
                &=\|z_k\|^2 + \sum_{i\neq k} z_i\cdot z_k \\
                &\geq \|z_k\|^2 - \sum_{i\neq k} |z_i\cdot z_k| \\
                &\geq \frac{p}{c} - cn \left(\|\mu\|^2 + \sqrt{p}\left(\log\frac{n}{\delta}\right)^\frac{1}{\alpha}\right) \\
                & \geq \frac{p}{c} - 2cn\max\left(\|\mu\|^2,   \sqrt{p}\left(\log\frac{n}{\delta}\right)^\frac{1}{\alpha}\right) \\
                & = \frac{1}{c}\left(p - 2c^2n\max\left(\|\mu\|^2,  \sqrt{p}\left(\log\frac{n}{\delta}\right)^\frac{1}{\alpha}\right)\right) .
            \end{align*}
            By assumptions (A3) and (A4), for large enough $C$ we have
            \begin{equation*}
                y_k v\cdot x_k > 0,
            \end{equation*}
            which completes our proof.
        \end{proof}
    \subsubsection{Proof of Lemma \ref{loss_ratio_sigmoid}}
        In this section, we will assume that samples satisfy all the conditions of Lemma \ref{concentration_lemma}. First, we will prove that the ratio of the losses between any pair of points is bounded. In this proof, we use Lemma \ref{g(z)}, and Lemma \ref{loss_ratio_sigmoid} is derived from Lemma \ref{g(z)} and Lemma \ref{loss_ratio_exp}.
        \begin{lem}\label{g(z)}
            For any $s_1, s_2\in\mathbb{R}$,
            \begin{equation*}
                \frac{1+\exp(s_2)}{1+\exp(s_1)}\leq \max\left(2, 2\frac{\exp(-s_1)}{\exp(-s_2)}\right).
            \end{equation*}
        \end{lem}
        
        \begin{lem}\label{loss_ratio_exp}
            There exists a positive constant $c_2$ such that, for all large enough $C$, and any learning rate $\beta$ which satisfies
            \begin{equation*}
                \beta \leq \frac{1}{2}\left( c_1 p + 2 n c_1 \left( \|\mu\|^2 + \sqrt{p} \left(\log\frac{n}{\delta}\right)^\frac{1}{\alpha}\right)\right)^{-1},
            \end{equation*}
            for all iterations $t\geq 0$,
            \begin{equation*}
                \max_{i,j\in [n]}\left\{\frac{\exp(-\theta^{(t)}\cdot z_i)}{\exp(-\theta^{(t)}\cdot z_j)}\right\} \leq c_2,
            \end{equation*}
            where $c_1$ is a constant which satisfies Lemma \ref{concentration_lemma}.
        \end{lem}
        
        \begin{proof}
            For simplicity, let $A_t$ be the ratio between exponential losses of the first and second samples for $t$ iterations:
            \begin{equation*}
                A_{t} = \frac{\exp(-\theta^{(t)}\cdot z_1)}{\exp(-\theta^{(t)}\cdot z_2)}.
            \end{equation*}
            We will show that $A_t\leq 4 c_1^2$ by using induction. When $t=0$, $A_0 = 1 \leq 4 c_1^2$. Thus, the base step holds. Assume that the inductive hypothesis holds for some iteration $t$. We shall now show that it must hold at iteration $t+1$.
            \begin{align*}
                A_{t+1} &= \frac{\exp(-\theta^{(t+1)}\cdot z_1)}{\exp(-\theta^{(t+1)}\cdot z_2)}\\
                &= \frac{\exp\left(-(\theta^{(t)}-\beta\nabla R(\theta^{(t)}))\cdot z_1\right)}{\exp\left(-(\theta^{(t)}-\beta\nabla R(\theta^{(t)}))\cdot z_2\right)}\\
                &= A_t \frac{\exp\left(\beta\nabla R(\theta^{(t)})\cdot z_1\right)}{\exp\left(\beta\nabla R(\theta^{(t)})\cdot z_2\right)}\\
                &= A_t \frac{\exp\left(-\beta\sum_{k\in[n]}\frac{z_k\cdot z_1}{1+ \exp(\theta^{(t)}\cdot z_k)}\right)}{\exp\left(-\beta\sum_{k\in[n]}\frac{z_k\cdot z_2}{1+ \exp(\theta^{(t)}\cdot z_k)}\right)}\\
                &= A_t \frac{\exp\left(-\beta\frac{\|z_1\|^2}{1+ \exp(\theta^{(t)}\cdot z_1)}\right)}{\exp\left(-\beta\frac{\|z_2\|^2}{1+ \exp(\theta^{(t)}\cdot z_2)}\right)}\frac{\exp\left(-\beta\sum_{k\neq 1}\frac{z_k\cdot z_1}{1+ \exp(\theta^{(t)}\cdot z_k)}\right)}{\exp\left(-\beta\sum_{k\neq 2}\frac{z_k\cdot z_2}{1+ \exp(\theta^{(t)}\cdot z_k)}\right)}\\
                 &= A_t \exp\left(-\beta\left(\frac{\|z_1\|^2}{1+ \exp(\theta^{(t)}\cdot z_1)}-\frac{\|z_2\|^2}{1+ \exp(\theta^{(t)}\cdot z_2)}\right)\right)\\
                 &\quad\times\exp\left(-\beta\left(\sum_{k\neq 1}\frac{z_k\cdot z_1}{1+ \exp(\theta^{(t)}\cdot z_k)}-\sum_{k\neq 2}\frac{z_k\cdot z_2}{1+ \exp(\theta^{(t)}\cdot z_k)}\right)\right).
            \end{align*}
            By Lemma \ref{concentration_lemma}, for any $k,i\neq j \in [n]$, there exists a constant $c_1$ such that
                \begin{gather*}
                    \frac{p}{c_1}\leq \|z_k\|^2\leq c_1 p,\\
                    |z_i\cdot z_j| \leq c_1 \left( \|\mu\|^2 + \sqrt{p} \left(\log\frac{n}{\delta}\right)^\frac{1}{\alpha}\right).
                \end{gather*}
            Thus, 
            \begin{align*}
                A_{t+1} &\leq A_{t}\exp\left(-\beta\left(\frac{p/c_1}{1+ \exp(\theta^{(t)}\cdot z_1)}-\frac{c_1 p}{1+ \exp(\theta^{(t)}\cdot z_2)}\right)\right)\\
                &\quad\times\exp\left(2\beta\sum_{k\in [n]}\frac{c_1 \left( \|\mu\|^2 + \sqrt{p} \left(\log\frac{n}{\delta}\right)^\frac{1}{\alpha}\right)}{1+ \exp(\theta^{(t)}\cdot z_k)}\right)\\
                &= A_{t}\exp\left(-\frac{\beta p}{c_1 (1+ \exp(\theta^{(t)}\cdot z_2))}\left(\frac{1+ \exp(\theta^{(t)}\cdot z_2)}{1+ \exp(\theta^{(t)}\cdot z_1)}-c_1^2\right)\right)\\
                &\quad\times\exp\left(2\beta\sum_{k\in [n]}\frac{c_1 \left( \|\mu\|^2 + \sqrt{p} \left(\log\frac{n}{\delta}\right)^\frac{1}{\alpha}\right)}{1+ \exp(\theta^{(t)}\cdot z_k)}\right).
            \end{align*}
            Now we consider two disjoint cases.\\
            \\
            \textbf{Case 1} ($A_t < 2 c_1^2$): 
            \begin{align*}
                A_{t+1} 
                &\leq A_{t}\exp\left(\frac{\beta c_1 p}{1+ \exp(\theta^{(t)}\cdot z_2)}\right)\times\exp\left(2\beta n  c_1 \left( \|\mu\|^2 + \sqrt{p} \left(\log\frac{n}{\delta}\right)^\frac{1}{\alpha}\right)\right)\\
                &\leq A_{t}\exp\left(\beta\left( c_1 p + 2 n c_1 \left( \|\mu\|^2 + \sqrt{p} \left(\log\frac{n}{\delta}\right)^\frac{1}{\alpha}\right)\right)\right).
            \end{align*}
            Taking $\beta$ small enough that
            \begin{equation*}
                 \beta \leq \frac{1}{2}\left( c_1 p + 2 n c_1 \left( \|\mu\|^2 + \sqrt{p} \left(\log\frac{n}{\delta}\right)^\frac{1}{\alpha}\right)\right)^{-1},
            \end{equation*}
            we have
            \begin{equation*}
                A_{t+1} \leq  A_t \exp\left(\frac{1}{2}\right) \leq  2c_1^2\exp\left(\frac{1}{2}\right)  < 4c_1^2.
            \end{equation*}
            \\
            \textbf{Case 2} ($A_t \geq 2 c_1^2$): 
            \begin{align*}
                A_{t+1} &= A_{t}\exp\left(-\frac{\beta  p}{c_1 (1+ \exp(\theta^{(t)}\cdot z_2))}\left(\frac{1+ \exp(\theta^{(t)}\cdot z_2)}{1+ \exp(\theta^{(t)}\cdot z_1)}-c_1 ^2\right)\right)\\
                &\quad\times\exp\left(2\beta c_1 \left( \|\mu\|^2 + \sqrt{p} \left(\log\frac{n}{\delta}\right)^\frac{1}{\alpha}\right)\frac{1}{1+\exp(\theta^{(t)}\cdot z_2)}\sum_{k\in [n]}\frac{1+\exp(\theta^{(t)}\cdot z_2)}{1+ \exp(\theta^{(t)}\cdot z_k)}\right).
            \end{align*}
            By Lemma \ref{g(z)} and the induction hypothesis, 
            \begin{align*}
                A_{t+1} &\leq A_{t}\exp\left(-\frac{\beta  c_1 p}{1+ \exp(\theta^{(t)}\cdot z_2)}\right)\\
                &\quad\times\exp\left(2\beta c_1 \left( \|\mu\|^2 + \sqrt{p} \left(\log\frac{n}{\delta}\right)^\frac{1}{\alpha}\right)\frac{1}{1+\exp(\theta^{(t)}\cdot z_2)}\sum_{k\in [n]}\max\left(2, 2 A_t\right)\right)\\
                &\leq A_{t}\exp\left(-\frac{\beta  c_1 p}{1+ \exp(\theta^{(t)}\cdot z_2)}\right)\\
                &\quad\times\exp\left(2\beta c_1 \left( \|\mu\|^2 + \sqrt{p} \left(\log\frac{n}{\delta}\right)^\frac{1}{\alpha}\right)\frac{n\max\left(2, 8c_1 ^2\right)}{1+\exp(\theta^{(t)}\cdot z_2)}\right)\\
                &\leq A_t\exp\left(-\frac{\beta c_1}{1+\exp(\theta^{(t)}\cdot z_2)}\left(p - 8 c_1 ^2 n\left( \|\mu\|^2 + \sqrt{p} \left(\log\frac{n}{\delta}\right)^\frac{1}{\alpha}\right)\right)\right).
            \end{align*}
            By assumptions (A3) and (A4), for large enough $C$,
            \begin{equation*}
                p - 8 c_1 ^2 n\left( \|\mu\|^2 + \sqrt{p} \left(\log\frac{n}{\delta}\right)^\frac{1}{\alpha}\right)>0.
            \end{equation*}
            Thus,
            \begin{equation*}
                A_{t+1} < A_t \leq 4c_1^2.
            \end{equation*}
            This completes the proof of the inductive step.
        \end{proof}
    \subsubsection{Proof of Lemma \ref{wmu}}
        \begin{proof}[Proof of Lemma \ref{wmu}]
            \begin{align*}
                \mu\cdot \theta^{(t+1)} &= \mu \cdot \theta^{(t)} + \beta\sum_{k\in[n]}\frac{\mu\cdot z_k}{1+\exp(\theta^{(t)}\cdot z_k)}\\
                &= \mu \cdot \theta^{(t)} + \beta\sum_{k\in\mathcal{C}}\frac{\mu\cdot z_k}{1+\exp(\theta^{(t)}\cdot z_k)} + \beta\sum_{k\in\mathcal{N}}\frac{\mu\cdot z_k}{1+\exp(\theta^{(t)}\cdot z_k)}.
            \end{align*}
            By Lemma \ref{concentration_lemma},
            \begin{align*}
                \mu\cdot \theta^{(t+1)} 
                &\geq \mu \cdot \theta^{(t)} + \frac{\beta\|\mu\|^2}{2}\sum_{k\in\mathcal{C}}\frac{1}{1+\exp(\theta^{(t)}\cdot z_k)} - \frac{3\beta\|\mu\|^2}{2}\sum_{k\in\mathcal{N}}\frac{1}{1+\exp(\theta^{(t)}\cdot z_k)}\\
                &\geq \mu \cdot \theta^{(t)} + \frac{\beta\|\mu\|^2}{2}\sum_{k\in[n]}\frac{1}{1 + \exp(\theta^{(t)}\cdot z_k)} - 2\beta\|\mu\|^2\sum_{k\in\mathcal{N}}\frac{1}{1+\exp(\theta^{(t)}\cdot z_k)}.
            \end{align*}
            By $|\mathcal{N}|\leq(\eta +c')n$ and Lemma \ref{loss_ratio_sigmoid}, 
            \begin{align*}
                \sum_{k\in\mathcal{N}}\frac{1}{1+\exp(\theta^{(t)}\cdot z_k)}&\leq c_3(\eta + c')n \min_{k\in[n]} \frac{1}{1+\exp(\theta^{(t)}\cdot z_k)}\\
                &\leq c_3(\eta + c') \sum_{k\in[n]}\frac{1}{1+\exp(\theta^{(t)}\cdot z_k)},
            \end{align*}
            where $c_3$ is the constant from Lemma \ref{loss_ratio_sigmoid}.
            Recalling that $\eta \leq \frac{1}{C}$ and $c'$ is an arbitrary positive constant, for large enough $C$ and small enough $c'$,
            \begin{equation*}
                \sum_{k\in\mathcal{N}}\frac{1}{1+\exp(\theta^{(t)}\cdot z_k)}\leq \frac{1}{8}\sum_{k\in[n]}\frac{1}{1+\exp(\theta^{(t)}\cdot z_k)}.
            \end{equation*}
            Thus, we have
            \begin{align*}
                \mu\cdot \theta^{(t+1)} 
                &\geq \mu\cdot \theta^{(t)} + \frac{\beta\|\mu\|^2}{4}\sum_{k\in[n]}\frac{1}{1 + \exp(\theta^{(t)}\cdot z_k)}.
            \end{align*}
            By using this inequality repeatedly and $\theta^{(0)}=0$, 
            \begin{align*}
                \mu\cdot \theta^{(t+1)} 
                &\geq \frac{\beta\|\mu\|^2}{4}\sum_{m=0}^t\sum_{k\in[n]}\frac{1}{1 + \exp(\theta^{(m)}\cdot z_k)},\\
                \|w\|\frac{\mu\cdot \theta^{(t+1)}}{\|\theta^{(t+1)}\|}
                &\geq \|w\|\frac{\beta\|\mu\|^2\sum_{m=0}^t\sum_{k\in[n]}\frac{1}{1 + \exp(\theta^{(m)}\cdot z_k)}}{4\|\theta^{(t+1)}\|}.
            \end{align*}
            By taking the large-$t$ limit and using Lemma \ref{Soudry},
            \begin{equation}\label{mucdotw}
                \mu\cdot w \geq \beta \|w\|\|\mu\|^2\lim_{t\rightarrow\infty}\frac{\sum_{m=0}^t\sum_{k\in[n]}\frac{1}{1 + \exp(\theta^{(m)}\cdot z_k)}}{4\|\theta^{(t+1)}\|}
            \end{equation}
            By definition of gradient descent iterations,
            \begin{align*}
                \|\theta^{(t+1)}\| &= \left\|\sum_{m=0}^t \beta \nabla R(\theta^{(m)})\right\|\\
                &\leq\beta\sum_{m=0}^t\|\nabla R(\theta^{(m)})\|\\
                &\leq\beta\sum_{m=0}^t\left\|\sum_{k\in[n]}\frac{- z_k}{1 + \exp(\theta^{(m)}\cdot z_k)}\right\|\\
                &\leq\beta c_1 \sqrt{p}\sum_{m=0}^t\sum_{k\in[n]}\frac{1}{1 + \exp(\theta^{(m)}\cdot z_k)}.
            \end{align*}
            With inequality (\ref{mucdotw}), we have
            \begin{equation*}
                \mu\cdot w \geq \frac{\|w\|\|\mu\|^2}{4 c_1 \sqrt{p}},
            \end{equation*}
            which completes our proof.
        \end{proof}
    \subsection{Proof of Proposition \ref{bounds_singular_value_2}}\label{proof_singlura_value}
        Let $\tilde{X}$ denote $[\tilde{y}_1\tilde{x}_1, \cdots, \tilde{y}_n\tilde{x}_n] \in \mathbb{R}^{p \times n}$ where $\tilde{x}_k = q_k + \mu \tilde{y}_k$. $\{q_k\}$ and $\{y_k\}$ are independent of each other. $q_k \sim P_\text{clust}$ and $\tilde{y}_k \sim \text{Uniform}\{-1,1\}$ for each $k \in [n]$. Let $\tilde{y} = [\tilde{y}_1, \dots, \tilde{y}_n]^T$. Before proving Proposition \ref{bounds_singular_value_2}, we first present Proposition \ref{bounds_singular_value} for a simpler case.
        \begin{prop}[A bound of the singular values of $\tilde{X}$]\label{bounds_singular_value}
            We assume $\sigma^2 = \mathbb{E}_{q_i \sim P_{\text{clust}}^{(i)}}[q_i^2]$ for any $i \in [n]$. For any $\delta \geq 0$, with probability at least $1-\delta$, there are constants $c_3$, $c_4$ depending on 
            $\alpha$ such that
            \begin{align*}
                s_1(\tilde{X}) \leq \sigma \sqrt{p} \left(1 +  \frac{2n\|\mu\|^2}{\sigma^2 p} + \frac{c_3+c_4\max_{i}|\mu_i|
                \sqrt{n}}{\sigma^2p}\left(n \log 9 + \log\frac{4}{\delta}\right)^\frac{2}{\alpha}\right).
            \end{align*}
            This bound also holds for \( X = [y_1\tilde{x}_1, \cdots, y_n\tilde{x}_n] \), which consists of labels \( y \) flipped with a certain probability \( \eta \) without depending on $\tilde{x}$.
        \end{prop}
        
        In the proof of Proposition \ref{bounds_singular_value}, we use the following lemmas.
        
        \begin{lem}[Corollary 4.2.13 in \cite{vershyninHighdimensionalProbabilityIntroduction2018}]\label{Net_lemma}
            The covering numbers of the Euclidean ball \( B_n^2 := \{x \in \mathbb{R}^n \mid \|x\| \leq 1\} \) satisfy the following for any \( \epsilon > 0 \):
            \begin{align*}
                \left(\frac{1}{\epsilon}\right)^n \leq \mathcal{N}(B_n^2, \epsilon) \leq \left(\frac{2}{\epsilon} + 1\right)^n.
            \end{align*}
            The same upper bound holds for the unit Euclidean sphere \( S^{n-1} \).
        \end{lem}
        
        \begin{lem}[Exercise 4.4.3 in \cite{vershyninHighdimensionalProbabilityIntroduction2018}]\label{Computing_operator_norm}
            Let \( A \) be an \( n \times n \) real symmetric matrix and \( \epsilon \in [0, 1/2) \). For any \( \epsilon \)-net \( \mathcal{N}_\epsilon \) of the sphere \( S^{n-1} \),
            \begin{align*}
                \underset{x \in \mathcal{N}_\epsilon}{\mathrm{sup}} |(Ax) \cdot x| \leq \|A\|_{\mathrm{op}} \leq \frac{1}{1-2\epsilon} \underset{x \in \mathcal{N}_\epsilon}{\mathrm{sup}} |(Ax) \cdot x|.
            \end{align*}
        \end{lem}
        \begin{lem}[Lemma A.3 in \cite{gotze2021concentration}]\label{sum_alpha}
            For any $\alpha\in(0, 1)$ and any random variables $X,Y$ we have
            \begin{align*}
                \|X+Y\|_{\psi_{\alpha}}\leq 2^{1/\alpha}(\|X\|_{\psi_{\alpha}}+\|Y\|_{\psi_{\alpha}}).
            \end{align*}
        \end{lem}
        \begin{lem}[Lemma 4.1.5 in \cite{vershyninHighdimensionalProbabilityIntroduction2018}]\label{Approximate_isometries}
            Let \( A \) be an \( m \times n \) real matrix and \( \delta > 0 \). Suppose that
            \begin{align*}
                \|A^{T}A - I_n\|_{\mathrm{op}} \leq \max(\epsilon, \epsilon^2).
            \end{align*}
            Then 
            \begin{align*}
                (1 - \epsilon) \|x\| \leq \|Ax\| \leq (1 + \epsilon) \|x\| \quad \text{for all } x \in \mathbb{R}^n.
            \end{align*}
            Consequently, 
            \begin{align*}
                1 - \epsilon \leq s_k(A) \leq 1 + \epsilon \quad \text{for all } k \in [n].
            \end{align*}
        \end{lem}
    
        \begin{proof}[Proof of Proposition \ref{bounds_singular_value}]
            By Lemma \ref{Net_lemma}, there exists a \( \frac{1}{4} \)-net \( \mathcal{N}_{1/4} \) of the unit sphere \( S^{n-1} \) with cardinality \( |\mathcal{N}_{1/4}| \leq 9^n \). By Lemma \ref{Computing_operator_norm}, we have
            \begin{align}
                \left\|\frac{1}{\sigma^2 p}\tilde{X}^{T}\tilde{X} - I_n\right\|_{\mathrm{op}} &\leq 2 \underset{u \in \mathcal{N}_{1/4}}{\max}\left|\left(\left(\frac{1}{\sigma^2 p}\tilde{X}^{T}\tilde{X} - I_n\right)u\right) \cdot u\right| \nonumber\\
                &= 2 \underset{u \in \mathcal{N}_{1/4}}{\max}\left|\frac{1}{\sigma^2 p}\|\tilde{X}u\|^2 - 1\right|.\label{leqmax}
            \end{align}
            Fix \( u \in S^{n-1} \). Let \( r_i \in \mathbb{R}^n \) denote the \( i \)-th row of \( Q = [q_1, \cdots, q_n] \in \mathbb{R}^{p \times n} \). We have
            \begin{align*}
                &\frac{1}{\sigma^2p}\|\tilde{X}u\|^2  \\
                &= \frac{1}{p}\sum_{i=1}^p \frac{1}{\sigma^2}((r_i \odot \tilde{y} + \mu_i \mathbf{1})\cdot u)^2\\
                &=\frac{1}{p}\sum_{i=1}^p \frac{1}{\sigma^2}\left(((r_i\odot\tilde{y}) \cdot u)^2 + 2 ((r_i\odot\tilde{y}) \cdot u)\sum_{j=1}^n \mu_i u_j + \mu_i^2\left(\sum_{j=1}^n  u_j\right)^2\right)\\
                &= \frac{1}{p}\sum_{i=1}^p \left(f_i + \frac{\mu_i^2}{\sigma^2}\left(\sum_{j=1}^n  u_j\right)^2\right),
            \end{align*}
            where $f_i =\frac{1}{\sigma^2}\left((r_i\cdot(\tilde{y}\odot u))^2 + 2 (r_i\cdot(\tilde{y}\odot u)) \mu_i\sum_{j=1}^n u_j\right)$.
            Thus,
            \begin{align*}
                &\mathbb{P}\left[\left|\frac{1}{\sigma^2 p}\|Xu\|^2 - 1\right| > \frac{\epsilon}{2}\right] \\
                &= \mathbb{P}\left[\left|\frac{1}{p}\sum_{i=1}^p \left(f_i + \frac{\mu_i^2}{\sigma^2}\left(\sum_{j=1}^n  u_j\right)^2\right) - 1\right| > \frac{\epsilon}{2}\right] \\
                &= \sum_{\tilde{y} \in \{-1, 1\}^n} \underset{Q \sim P_{\mathrm{clust}}^{n}}{\mathbb{P}}\left[\left|\frac{1}{p}\sum_{i=1}^p \left(f_i + \frac{\mu_i^2}{\sigma^2}\left(\sum_{j=1}^n  u_j\right)^2\right) - 1\right| > \frac{\epsilon}{2}\right] 2^{-n} \\
                &= \sum_{\tilde{y} \in \{-1, 1\}^n} \underset{Q \sim P_{\mathrm{clust}}^{n}}{\mathbb{P}}\left[\frac{1}{p}\sum_{i=1}^p f_i  - 1 > \frac{\epsilon}{2} - \frac{\|\mu\|^2}{\sigma^2 p}\left(\sum_{j=1}^n  u_j\right)^2\right] 2^{-n} \\
                &\quad+ \sum_{\tilde{y} \in \{-1, 1\}^n} \underset{Q \sim P_{\mathrm{clust}}^{n}}{\mathbb{P}}\left[\frac{1}{p}\sum_{i=1}^p f_i  - 1  < -\frac{\epsilon}{2} - \frac{\|\mu\|^2}{\sigma^2 p}\left(\sum_{j=1}^n  u_j\right)^2\right] 2^{-n}.\\
                &\leq \sum_{\tilde{y} \in \{-1, 1\}^n} \underset{Q \sim P_{\mathrm{clust}}^{n}}{\mathbb{P}}\left[\frac{1}{p}\sum_{i=1}^p f_i  - 1 > \frac{\epsilon}{2} - \frac{n\|\mu\|^2}{\sigma^2 p}\right] 2^{-n} \\
                &\quad+ \sum_{\tilde{y} \in \{-1, 1\}^n} \underset{Q \sim P_{\mathrm{clust}}^{n}}{\mathbb{P}}\left[\frac{1}{p}\sum_{i=1}^p f_i  - 1  < -\frac{\epsilon}{2}\right] 2^{-n}.
            \end{align*}
            The final inequality was derived using $\sum_{j=1}^n u_j \leq \sqrt{n}$.
            \( \underset{Q \sim P_{\mathrm{clust}}^{n}}{\mathbb{E}}[f_i] = 1 \) and \( \{f_i\}_{i=1}^p \) are independent random variables 
            when conditioned on $\tilde{y}$. All elements of $Q$ are $\alpha$ sub-exponential with their exponential Orlicz norm at most $1$. By Proposition \ref{sum_concentration_inequality}, there is a constant $c$ depending on $\alpha$ such that for any $t\geq 0$,
            \begin{align*}
                \mathbb{P}[|r_i\cdot(\tilde{y}\odot u) |\geq t]&\leq 2\exp\left(-c\frac{t^\alpha}{\|\tilde{y}\odot u\|^\alpha}\right) =2\exp\left(-ct^\alpha\right).
            \end{align*} 
            Thus, there is a constant $K_1$ such that
            \begin{align*}
                \|r_i\cdot(\tilde{y}\odot u)\|_{\psi_\alpha}\leq K_1.
            \end{align*} 
            Then, we have
            \begin{align*}
                \|(r_i\cdot(\tilde{y}\odot u))^2\|_{\psi_{\alpha/2}} & = \inf\left\{t>0 : \mathbb{E}\left[\exp\left(\frac{|r_i\cdot(\tilde{y}\odot u)|^{\alpha}}{\sqrt{t}^{\alpha}}\right)\right]\leq 2\right\}
                \leq \sqrt{K_1},
            \end{align*} 
            and there is a constant $K_2$ such that
            \begin{align*}
                \|r_i\cdot(\tilde{y}\odot u)\|_{\psi_{\alpha/2}} \leq K_2.
            \end{align*} 
            By Lemma \ref{sum_alpha} and the fact that $\|\cdot\|_{\psi_1}$ preserves the triangle inequality, we obtain 
            \begin{align*}
                \|f_i\|_{\psi_{\alpha/2}} &= \frac{1}{\sigma^2}\left\|(r_i\cdot(\tilde{y}\odot u))^2 + 2 (r_i\cdot(\tilde{y}\odot u)) \mu_i\sum_{j=1}^n u_j\right\|_{\psi_{\alpha/2}}\\
                &\leq \frac{2^{2/\alpha}}{\sigma^2}\left(\left\|(r_i\cdot(\tilde{y}\odot u))^2 \right\|_{\psi_{\alpha/2}}+\left\|2 (r_i\cdot(\tilde{y}\odot u))\mu_i\sum_{j=1}^n  u_j\right\|_{\psi_{\alpha/2}}\right)\\
                &=\frac{2^{2/\alpha}}{\sigma^2}\left(\sqrt{K_1}+\left|2\mu_i\
                \sum_{j=1}^nu_j\right|K_2\right)\\
                &\leq \frac{2^{2/\alpha}}{\sigma^2}\left(\sqrt{K_1}+2\max_{i}|\mu_i|
                \sqrt{n}K_2\right).
            \end{align*}
            Setting \( \frac{\epsilon}{2} \geq \frac{n\|\mu\|^2}{\sigma^2 p} \), by Proposition \ref{sum_concentration_inequality}, there exists a constant \( c \) such that
            \begin{align*}
                \mathbb{P}\left[\left|\frac{1}{\sigma^2 p}\|\tilde{X}u\|^2 - 1\right| > \frac{\epsilon}{2}\right] 
                &\leq \sum_{\tilde{y} \in \{-1, 1\}^n} 2 \exp\left(-\frac{\sigma^\alpha}{2c}\frac{\left(\frac{\epsilon}{2} - \frac{n\|\mu\|^2}{\sigma^2 p}\right)^{\alpha/2} p^{\alpha/2}}{\left(\sqrt{K_1}+2\max_{i}|\mu_i|
                \sqrt{n}K_2\right)^{\alpha/2}}\right) 2^{-n} \\
                &\quad+ \sum_{\tilde{y} \in \{-1, 1\}^n} 2 \exp\left(-\frac{\sigma^\alpha}{2c}\frac{\left(\frac{\epsilon}{2}\right)^{\alpha/2} p^{\alpha/2}}{\left(\sqrt{K_1}+2\max_{i}|\mu_i|
                \sqrt{n}K_2\right)^{\alpha/2}}\right) 2^{-n}.
            \end{align*}
        
            By setting 
            \[
            \frac{\epsilon}{2} = \frac{n\|\mu\|^2}{\sigma^2 p} + \frac{(2c)^{2/\alpha}\left(\sqrt{K_1}+2\max_{i}|\mu_i|
                \sqrt{n}K_2\right)}{\sigma^2p}\left(n \log 9 + \log\frac{4}{\delta}\right)^\frac{2}{\alpha},
            \]
            we have
            \begin{align*}
                \mathbb{P}\left[\left|\frac{1}{\sigma^2 p}\|\tilde{X}u\|^2 - 1\right| > \frac{\epsilon}{2}\right] &\leq \sum_{\tilde{y} \in \{-1, 1\}^n} 4 \exp\left(-n \log 9 - \log \frac{4}{\delta}\right) 2^{-n} \\
                &= 4 \exp\left(-n \log 9 - \log \frac{4}{\delta}\right).
            \end{align*}
        
            By inequality (\ref{leqmax}) and the union bound method, we have
            \begin{align*}
                \left\|\frac{1}{\sigma^2 p} \tilde{X}^{T}\tilde{X} - I_n\right\|_{\mathrm{op}} &\leq 2 \underset{u \in \mathcal{N}_{1/4}}{\max}\left|\left(\left(\frac{1}{\sigma^2 p}\tilde{X}^{T}\tilde{X} - I_n\right)u\right) \cdot u\right| \\
                & \leq\mathbb{P}\left[\underset{u \in \mathcal{N}_{1/4}}{\max}\left|\frac{1}{\sigma^2 p}\|\tilde{X}u\|^2 - 1\right| > \frac{\epsilon}{2}\right] \\
                &\leq 9^n \cdot 4 \exp\left(-n \log 9 - \log \frac{4}{\delta}\right) \\
                &= \delta.
            \end{align*}
        
            By Lemma \ref{Approximate_isometries}, we conclude that
            \begin{align*}
                &s_1(\tilde{X}) \\
                &\leq \sigma \sqrt{p} \left(1 +  \frac{2n\|\mu\|^2}{\sigma^2 p} + \frac{2(2c)^{2/\alpha}\left(\sqrt{K_1}+2\max_{i}|\mu_i|
                \sqrt{n}K_2\right)}{\sigma^2p}\left(n \log 9 + \log\frac{4}{\delta}\right)^{2/\alpha}\right),
            \end{align*}
            which completes our proof.
            A similar argument holds for \( X = [y_1\tilde{x_1}, \cdots, y_n\tilde{x_n}] \), which consists of labels \( y \) flipped with a certain probability \( \eta \) without depending on $\tilde{x}$.
        \end{proof}
        We will prove Proposition \ref{bounds_singular_value_2}, which provides an upper bound for the singular values of $X$, by making a slight modification to the proof of Proposition \ref{bounds_singular_value}. In the proof of Proposition \ref{bounds_singular_value_2} , we use the following lemma.
        \begin{lem}[Lemma A.2 in \cite{gotze2021concentration} and Proposition 8.1 in \cite{2010growth}]\label{LPnorm}
            Let $d_\alpha := \frac{(\alpha e)^{1/\alpha}}{2}$ and $D_\alpha := (2e)^{1/\alpha}$ for $\alpha \in (0,1)$, and let $d_\alpha := \frac{1}{2}$ and $D_\alpha := 2e$ for $\alpha \geq 1$. For any $\alpha > 0$ and any random variable $X$, we have
            \begin{align*}
                d_\alpha \sup_{p \geq 1} \|X\|_{L_p}    \leq \|X\|_{\psi_\alpha} \leq D_\alpha \sup_{p \geq 1} \|X\|_{L_p}.
            \end{align*}
        \end{lem}
        \begin{proof}[Proof of Proposition \ref{bounds_singular_value_2}]
             Fix \( u \in S^{n-1} \). Let \( r_i \in \mathbb{R}^n \) denote the \( i \)-th row of \( Q = [q_1, \cdots, q_n] \in \mathbb{R}^{p \times n} \). We have
            \begin{align*}
                &\frac{1}{p}\|Xu\|^2  \\
                &= \frac{1}{p}\sum_{i=1}^p ((r_i \odot y + \mu_i y\odot\tilde{y})\cdot u)^2\\
                &=\frac{1}{p}\sum_{i=1}^p \left(((r_i\odot y) \cdot u)^2 + 2 \mu_i ((r_i\odot y) \cdot u)  ((y\odot\tilde{y})\cdot u) + \mu_i^2\left(\sum_{j=1}^n  y_j \tilde{y}_j   u_j\right)^2\right)\\
                &= \frac{1}{p}\sum_{i=1}^p \left(f_i + \mu_i^2\left(\sum_{j=1}^n  y_j \tilde{y}_j u_j\right)^2\right),
            \end{align*}
             where $f_i =((r_i\odot y)\cdot u)^2 + 2\mu_i ((r_i\odot y)\cdot u) ((y\odot\tilde{y})\cdot u)$. Since $y\in\{-1,1\}$, each component of $r_i\odot y$ is $\alpha$ sub-exponential with their exponential Orlicz norm at most $1$. By Proposition \ref{sum_concentration_inequality}, there is a constant $c$ depending on $\alpha$ such that for any $t\geq 0$,
            \begin{align*}
                \mathbb{P}[|(r_i\odot y)\cdot u |\geq t]&\leq 2\exp\left(-c\frac{t^\alpha}{\| u\|^\alpha}\right) =2\exp\left(-ct^\alpha\right).
            \end{align*} 
            Thus, there is a constant $K_1$ such that
            \begin{align*}
                \|(r_i\odot y)\cdot u\|_{\psi_\alpha}\leq K_1,
            \end{align*} 
            and 
            \begin{align*}
                \|2\mu_i ((r_i\odot y)\cdot u)((y\odot\tilde{y})\cdot u)\|_{\psi_\alpha} &= 2|\mu_i|\left\|((r_i\odot y)\cdot u)((y\odot\tilde{y})\cdot u)\right\|_{\psi_\alpha}\\
                &\leq 2|\mu_i|\left\||(r_i\odot y)\cdot u||(y\odot\tilde{y})\cdot u|\right\|_{\psi_\alpha}\\\
                &\leq 2|\mu_i|\sum_{j=1}^n|u_j|\left\|(r_i\odot y)\cdot u\right\|_{\psi_\alpha}\\
                & \leq 2 K_1 |\mu_i|\sum_{j=1}^n|u_j|.
            \end{align*}
            By Lemma \ref{LPnorm}, there is a constant $K_2$ and $K_3$ such that for any $p\geq 1$,
            \begin{gather*}
                \|(r_i\odot y)\cdot u\|_{L_p} \leq K_3 p^{1/\alpha},\\
                \|2\mu_i ((r_i\odot y)\cdot u)((y\odot\tilde{y})\cdot u)\|_{L^p}\leq K_2 |\mu_i|\sum_{j=1}^n|u_j| p^{1/\alpha},
            \end{gather*}
            Therefore,
            \begin{gather*}
                \mathbb{E}[((r_i\odot y)\cdot u)^2] \in[0, 2^{1/\alpha}K_3],\\
                \mathbb{E}[2\mu_i ((r_i\odot y)\cdot u)((y\odot\tilde{y})\cdot u)] \in \left[-K_2 |\mu_i|\sum_{j=1}^n|u_j| , K_2 |\mu_i|\sum_{j=1}^n|u_j| \right].
            \end{gather*}
            By combining these two, we have
            \begin{align*}
                \mathbb{E}[f_i] \in \left[-K_2|\mu_i|\sum_{j=1}^n|u_j| , 2^{1/\alpha}K_3 +K_2 |\mu_i|\sum_{j=1}^n|u_j| \right].
            \end{align*}
            Thus,
            \begin{align*}
                &\mathbb{P}\left[\left|\frac{1}{p}\|Xu\|^2 - 1\right| > \frac{\epsilon}{2}\right] \\
                &= \mathbb{P}\left[\left|\frac{1}{p}\sum_{i=1}^p \left(f_i + \mu_i^2\left(\sum_{j=1}^n  y_j \tilde{y}_j u_j\right)^2\right) - 1\right| > \frac{\epsilon}{2}\right] \\
                &=\mathbb{P}\left[\frac{1}{p}\sum_{i=1}^p (f_i  - \mathbb{E}[f_i])  > \frac{\epsilon}{2} + 1 -\frac{1}{p}\sum_{i=1}^p\mathbb{E}[f_i] - \frac{\|\mu\|^2}{p}\left(\sum_{j=1}^n  y_j \tilde{y}_j u_j\right)^2\right] \\
                &\quad+\mathbb{P}\left[\frac{1}{p}\sum_{i=1}^p( f_i  -\mathbb{E}[f_i])  < -\frac{\epsilon}{2}+1 -\frac{1}{p}\sum_{i=1}^p\mathbb{E}[f_i]  - \frac{\|\mu\|^2}{p}\left(\sum_{j=1}^n  y_j \tilde{y}_j u_j\right)^2\right]\\
                &\leq\mathbb{P}\left[\frac{1}{p}\sum_{i=1}^p (f_i  - \mathbb{E}[f_i])  > \frac{\epsilon}{2} -\left(2^{1/\alpha}K_3+\frac{K_2}{p} \sum_{i=1}^p|\mu_i|\sum_{j=1}^n|u_j|\right)  - \frac{n\|\mu\|^2}{p}\right] \\
                &\quad+\mathbb{P}\left[\frac{1}{p}\sum_{i=1}^p( f_i  -\mathbb{E}[f_i])  < -\frac{\epsilon}{2}+1 -\left(-\frac{K_2}{p} \sum_{i=1}^p|\mu_i|\sum_{j=1}^n|u_j| \right) \right]\\
                &\leq\mathbb{P}\left[\frac{1}{p}\sum_{i=1}^p (f_i  - \mathbb{E}[f_i])  > \frac{\epsilon}{2} -\left(2^{1/\alpha}K_3+\frac{K_2}{p} \sqrt{n}\sum_{i=1}^p|\mu_i|\right)  - \frac{n\|\mu\|^2}{p}\right] \\
                &\quad+\mathbb{P}\left[\frac{1}{p}\sum_{i=1}^p( f_i  -\mathbb{E}[f_i])  < -\frac{\epsilon}{2}+1 + \frac{K_2}{p} \sqrt{n}\sum_{i=1}^p|\mu_i| \right].
            \end{align*}
            By the same method as the proof of Proposition \ref{bounds_singular_value}, we have
            \begin{align*}
                \|f_i\|_{\psi_{\alpha/2}}\leq 2^{2/\alpha}\left(\sqrt{K_1}+2\max_{i}|\mu_i|
                \sqrt{n}K_2\right).
            \end{align*}
            Setting \(\frac{\epsilon}{2} \geq 2^{1/\alpha}K_3+\frac{K_2}{p} \sqrt{n}\sum_{i=1}^p|\mu_i| + \frac{n\|\mu\|^2}{p}+ 1 \), by Proposition \ref{sum_concentration_inequality}, there exists a constant \( c \) such that
            \begin{align*}
                &\mathbb{P}\left[\left|\frac{1}{p}\|\tilde{X}u\|^2 - 1\right| > \frac{\epsilon}{2}\right] \\
                &\leq \sum_{\tilde{y} \in \{-1, 1\}^n} 2 \exp\left(-\frac{1}{2c}\frac{\left(\frac{\epsilon}{2}  -\left(2^{1/\alpha}K_3+\frac{K_2}{p} \sqrt{n}\sum_{i=1}^p|\mu_i|\right)  - \frac{n\|\mu\|^2}{p}\right)^{\alpha/2} p^{\alpha/2}}{\left(\sqrt{K_1}+2\max_{i}|\mu_i|
                \sqrt{n}K_2\right)^{\alpha/2}}\right) 2^{-n} \\
                &\quad+ \sum_{\tilde{y} \in \{-1, 1\}^n} 2 \exp\left(-\frac{1}{2c}\frac{\left(\frac{\epsilon}{2} - 1 -\frac{K_2}{p} \sqrt{n}\sum_{i=1}^p|\mu_i|\right)^{\alpha/2} p^{\alpha/2}}{\left(\sqrt{K_1}+2\max_{i}|\mu_i|
                \sqrt{n}K_2\right)^{\alpha/2}}\right) 2^{-n}.
            \end{align*}

            By setting 
            \begin{align*}
            \frac{\epsilon}{2} &= 2^{1/\alpha}K_3+\frac{K_2}{p} \sqrt{n}\sum_{i=1}^p|\mu_i| + \frac{n\|\mu\|^2}{p}+ 1  \\
            &\quad+\frac{(2c)^{2/\alpha}\left(\sqrt{K_1}+2\max_{i}|\mu_i|
                \sqrt{n}K_2\right)}{p}\left(n \log 9 + \log\frac{4}{\delta}\right)^\frac{2}{\alpha},
            \end{align*}
            we have
            \begin{align*}
                \mathbb{P}\left[\left|\frac{1}{ p}\|Xu\|^2 - 1\right| > \frac{\epsilon}{2}\right] 
                &\leq 4 \exp\left(-n \log 9 - \log \frac{4}{\delta}\right).
            \end{align*}
            The remainder of the proof is the same as the proof of Proposition \ref{bounds_singular_value}.
        \end{proof}
    \subsection{Proofs of Corollary \ref{order_of_beta_by_p} and \ref{order_of_beta_by_n}}\label{Proof_of_cor}
        \begin{proof}[Proof of Corollary \ref{order_of_beta_by_p} and \ref{order_of_beta_by_n}]
            We have
            \begin{align*}
                \sum_{i=1}^p|\mu_i| &\leq \sqrt{p}\|\mu\| ,
            \end{align*}
            and
            \begin{align*}
                \max_{i}|\mu_i| \leq \|\mu\|.
            \end{align*}
            Under assumptions (A3) and (A4), we have the following inequalities:
            \begin{itemize}
                \item $p \geq C \|\mu\|^2 n$.
                \item $p \geq C \|\mu\| n^{3/2}\left(\log\frac{n}{\delta}\right)^{\frac{1}{\alpha}}$.
                \item $p \geq C n^2\left(\log\frac{n}{\delta}\right)^{\frac{2}{\alpha}}$.
            \end{itemize}
            Therefore, we have
            \begin{align*}
                &c_5 + \frac{c_6\sqrt{n}}{p} \sum_{i=1}^p|\mu_i| + \frac{2n\|\mu\|^2}{p} \\
                &\quad +\frac{c_7 + c_8 \max_{i}|\mu_i|\sqrt{n}}{p}\left(n \log 9 + \log\frac{4}{\delta}\right)^{\frac{2}{\alpha}} \\
                &\leq c_5 + \frac{c_6\sqrt{n}\|\mu\|}{\sqrt{p}} + \frac{2n\|\mu\|^2}{p} \\
                &\quad +\frac{c_7 + c_8 \|\mu\|\sqrt{n}}{p}\left(n \log 9 + \log\frac{4}{\delta}\right)^{\frac{2}{\alpha}} \\
                &\leq c_5 + \frac{c_6}{C} n^{-1}\left(\log\frac{n}{\delta}\right)^{-\frac{1}{\alpha}} + \frac{2}{C} \\
                &\quad +\frac{c_7}{C} n^{-2}\left(\log\frac{n}{\delta}\right)^{-\frac{2}{\alpha}} + \frac{c_8}{C}n^{-1}\left(\log\frac{n}{\delta}\right)^{-\frac{1}{\alpha}}\left(n \log 9 + \log\frac{4}{\delta}\right)^{\frac{2}{\alpha}} \\
                &= O\left(1 + n^{\frac{2}{\alpha}-1}(\log n)^{-\frac{1}{\alpha}}\right).
            \end{align*}
            By performing a similar calculation, we obtain
            \begin{align*}
                1 + \frac{2n}{p} \left( \|\mu\|^2 + \sqrt{p} \left(\log\frac{n}{\delta}\right)^{\frac{1}{\alpha}}\right) = O(1). 
            \end{align*}
            Therefore, if we regard $n$ as fixed, we have
            \begin{align*}
                \min \Bigg(\frac{8}{p} &\Bigg(c_5 + \frac{c_6\sqrt{n}}{p} \sum_{i=1}^p|\mu_i| + \frac{2n\|\mu\|^2}{p} \\
                    & \quad + \frac{c_7 + c_8 \max_{i}|\mu_i|\sqrt{n}}{p}\left(n \log 9 + \log\frac{4}{\delta}\right)^{\frac{2}{\alpha}}\Bigg)^{-2}, \\
                    &\frac{1}{c_2 p} \left(1 + \frac{2n}{p} \left( \|\mu\|^2 + \sqrt{p} \left(\log\frac{n}{\delta}\right)^{\frac{1}{\alpha}}\right)\right)^{-1} \Bigg) \\
                    &= O(p).
            \end{align*}
            On the other hand, if we regard $n$ as not fixed, we have
            \begin{align*}
                \min \Bigg(\frac{8}{p} &\Bigg(c_5 + \frac{c_6\sqrt{n}}{p} \sum_{i=1}^p|\mu_i| + \frac{2n\|\mu\|^2}{p} \\
                    & \quad + \frac{c_7 + c_8 \max_{i}|\mu_i|\sqrt{n}}{p}\left(n \log 9 + \log\frac{4}{\delta}\right)^{\frac{2}{\alpha}}\Bigg)^{-2}, \\
                    &\frac{1}{c_2 p} \left(1 + \frac{2n}{p} \left( \|\mu\|^2 + \sqrt{p} \left(\log\frac{n}{\delta}\right)^{\frac{1}{\alpha}}\right)\right)^{-1} \Bigg) \\
                    &= O\left(p^{-1}\left(1 + n^{\frac{2}{\alpha}-1}(\log n)^{-\frac{1}{\alpha}}\right)^{-2}\right).
            \end{align*}
        \end{proof}
    \section{Details of the Figures in the introduction}\label{appendix:intro}
    \subsection{Figure \ref{fig:tail_index} : An example of input in image analysis exhibiting heavier tails than sub-gaussian}\label{appendix:intro_xi}
        In this section, we demonstrate that some of the feature representations derived from real-world image datasets exhibit distributions heavier than sub-gaussian distributions.
        \subsubsection{Methodology}
            To estimate the tail-heaviness of feature distributions, we followed the steps below:
            \begin{enumerate}
                \item A total of $n$ samples were collected from intermediate layers of CNN models trained on several image datasets.
                \item Each sample of feature value was centered by subtracting its mean, and the absolute value of the result was taken.
                \item The upper 5\% of the sorted absolute values was selected, yielding order statistics \( (x^{(1)}, x^{(2)}, \dots, x^{(\lfloor0.05n\rfloor)}) \).
                \item For each \( x^{(i)} \), we computed the corresponding \( z_i = -\log(i/n) \), which approximates \( -\log \mathbb{P}[|X| \geq x^{(i)} ]\).
                \item We then performed a regression of \( (x^{(i)}, z_i) \) against the form \( f(t) = at^\xi + b \), enabling us to estimate the tail parameter \( \xi \). The regression was performed using the non-linear least squares method implemented via the \texttt{scipy.optimize.curve\_fit} function.
            \end{enumerate}
            The tail parameter \( \xi \) is crucial as it characterizes the heaviness of the distribution's tail, where  \( P(|X| \geq t) =\exp(-(at^\xi + b)) \).
            \begin{table}[t]
                \caption{Summary of datasets used in the simulation.}
                \label{table:datasets}
                \centering
                \begin{tabular}{|c|c|c|c|c|}
                \hline
                \textbf{Dataset}      & \textbf{Split}    & \textbf{Number of Images} & \textbf{Image Size} & \textbf{Number of Labels} \\ \hline
                \textbf{CIFAR-10 \cite{krizhevsky2009learning}}     & Training          & 50,000                    & 32x32               & 10                        \\ \cline{2-5} 
                                      & Test              & 10,000                    & 32x32               & 10                        \\ \hline
                \textbf{CIFAR-100 \cite{krizhevsky2009learning}}    & Training          & 50,000                    & 32x32               & 100                       \\ \cline{2-5} 
                                      & Test              & 10,000                    & 32x32               & 100                       \\ \hline
                \textbf{Fashion-MNIST \cite{xiao2017fashion}} & Training         & 60,000                    & 28x28               & 10                        \\ \cline{2-5} 
                                      & Test              & 10,000                    & 28x28               & 10                        \\ \hline
                \textbf{SVHN \cite{netzer2011reading}}         & Training          & 73,257                    & 32x32               & 10                        \\ \cline{2-5} 
                                      & Test              & 26,032                    & 32x32               & 10                        \\ \hline
                \end{tabular}
            \end{table}
            
                \begin{table}[h]
                    \centering
                    \caption{Mean and Variance of Estimated Tail Index ($\xi$)}\label{table:xi_estimation}
                    \scriptsize
                    \begin{tabular}{|l|c|c|}
                    \hline
                    Dataset & Mean ($\hat{\xi}$) & Variance ($\hat{\xi}$) \\
                    \hline
                    CIFAR-10 & 0.9771 & 0.1111 \\
                    CIFAR-100 & 1.0423 & 0.1281 \\
                    Fashion-MNIST & 1.4748 & 0.7295 \\
                    SVHN & 0.9272 & 0.0514 \\
                    Gaussian & 1.6172 & 0.2388 \\
                    Exponential & 0.8996 & 0.0535 \\
                    \hline
                    \end{tabular}
                    \end{table}
                    \begin{figure}[h]
                    \centering
                    \includegraphics[width=0.5\linewidth]{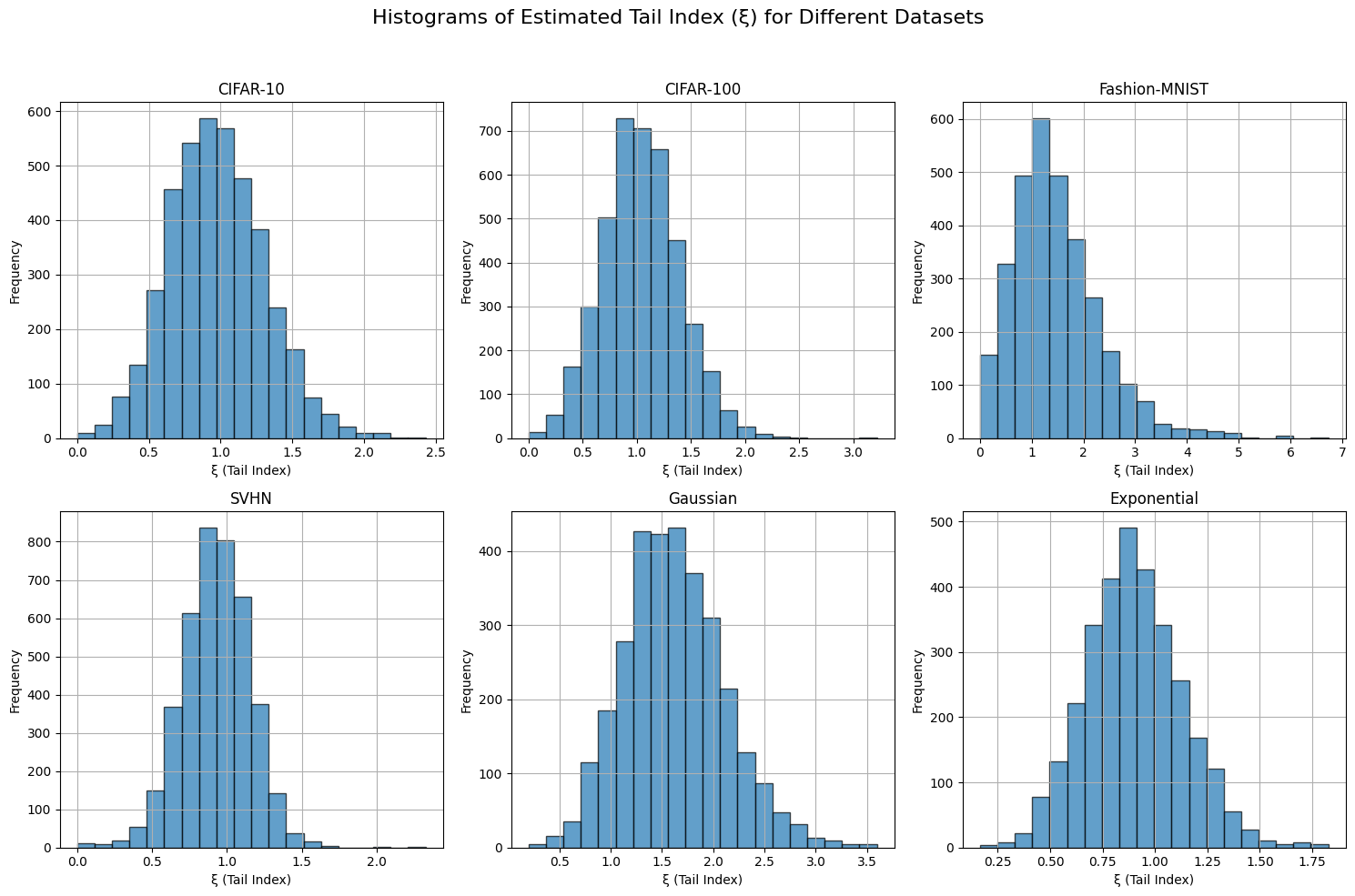}
                    \caption{Histograms of estimated tail index ($\xi$)}
                    \label{hist_xi}
                  \end{figure}
            \subsubsection{Datasets and Models}
            For the simulations, we used intermediate layer outputs from CNN models trained on the datasets listed in Table \ref{table:datasets}.
            
            The CNN architecture used for these datasets consisted of three convolutional layers (with 32, 64, and 64 filters, respectively) followed by max-pooling layers. The final fully connected layers had 64 neurons, with ReLU activations throughout the network. The total number of parameters in the network is typical for small-scale models. The output layer size was adjusted according to the number of classes in each dataset (e.g., 100 for CIFAR-100). This CNN model was trained using the Adam optimizer with a learning rate of 0.001 and cross-entropy loss as the loss function. Training was performed over 100 epochs, with a batch size of 100. All images were normalized and converted to PyTorch tensors prior to training.
            
            The CNN models were trained using the training split of the datasets listed in Table \ref{table:datasets}. After training, the test data from each dataset was used to generate intermediate layer outputs, which were then used in our analysis to evaluate the feature distributions.
            
            To compare these real-world results, we also generated samples from Gaussian and exponential distributions as baseline comparisons, aligning the future vector size with the smallest intermediate output size in this simulation, which is $3136$, and matching the sample size to the smaller value of $10000$.
            
        \subsubsection{Results}
            
            Table \ref{table:xi_estimation} and Figures \ref{fig:tail_index} and \ref{hist_xi} revealed that some intermediate layer outputs from CNN models exhibit distributions with heavier tails than Gaussian distributions. This finding reveals the necessity of developing a theory that addresses heavy-tailed distributions.
            
            It should be noted that, in the case of the Gaussian distribution, the reason why the value of $\xi$ is distributed below $2$ is that, in the samples used for the calculation, the approximation $P(|X| > t) \simeq 2\exp(-t^2/2)$ is not sufficiently accurate. When limited to the samples further in the tail of the distribution, the values of $\xi$ approach $2$.

        \subsection{Figure \ref{fig:gammma205025} : Benign overfitting can occur even for heavy inputs}\label{appendix:intro_benign}

            In this section, we conduct simulations to demonstrate that benign overfitting can occur even in settings with input distributions heavier-tailed than sub-gaussian. Specifically, we analyze the performance of a linear classifier trained using gradient descent on data drawn from generalized normal distributions, investigating the relationship between dimensionality \( p \), tail heaviness (controlled by the shape parameter \( \gamma \)), and the classifier's error rates.
            \subsubsection{Data Generation}
            \begin{table}[p]
\centering
\caption{Test Error Data with Mean and Standard Error of the Mean for Different $\gamma$ Values}\label{table:gamma_bo}
\scriptsize
\begin{tabular}{|c|c|c|c|}
\hline
Dimension (p) & $\gamma$ & Mean Test Error & Test Error SEM \\ \hline
100  & 0.25  & 0.1288 & 0.0047 \\ 
     & 0.5   & 0.1047 & 0.0049 \\ 
     & 2     & 0.0823 & 0.0040 \\ \hline
200  & 0.25  & 0.0808 & 0.0035 \\ 
     & 0.5   & 0.0660 & 0.0020 \\ 
     & 2     & 0.0627 & 0.0021 \\ \hline
300  & 0.25  & 0.0652 & 0.0018 \\ 
     & 0.5   & 0.0562 & 0.0014 \\ 
     & 2     & 0.0531 & 0.0012 \\ \hline
400  & 0.25  & 0.0578 & 0.0015 \\ 
     & 0.5   & 0.0519 & 0.0012 \\ 
     & 2     & 0.0531 & 0.0009 \\ \hline
500  & 0.25  & 0.0556 & 0.0013 \\ 
     & 0.5   & 0.0510 & 0.0009 \\ 
     & 2     & 0.0514 & 0.0008 \\ \hline
600  & 0.25  & 0.0525 & 0.0010 \\ 
     & 0.5   & 0.0493 & 0.0009 \\ 
     & 2     & 0.0504 & 0.0011 \\ \hline
700  & 0.25  & 0.0513 & 0.0007 \\ 
     & 0.5   & 0.0483 & 0.0009 \\ 
     & 2     & 0.0515 & 0.0012 \\ \hline
800  & 0.25  & 0.0518 & 0.0011 \\ 
     & 0.5   & 0.0495 & 0.0011 \\ 
     & 2     & 0.0512 & 0.0012 \\ \hline
900  & 0.25  & 0.0502 & 0.0009 \\ 
     & 0.5   & 0.0503 & 0.0010 \\ 
     & 2     & 0.0500 & 0.0009 \\ \hline
1000 & 0.25  & 0.0501 & 0.0010 \\ 
     & 0.5   & 0.0508 & 0.0009 \\ 
     & 2     & 0.0508 & 0.0009 \\ \hline
1100 & 0.25  & 0.0526 & 0.0011 \\ 
     & 0.5   & 0.0511 & 0.0009 \\ 
     & 2     & 0.0512 & 0.0011 \\ \hline
1200 & 0.25  & 0.0518 & 0.0009 \\ 
     & 0.5   & 0.0499 & 0.0007 \\ 
     & 2     & 0.0507 & 0.0010 \\ \hline
1300 & 0.25  & 0.0499 & 0.0010 \\ 
     & 0.5   & 0.0498 & 0.0010 \\ 
     & 2     & 0.0498 & 0.0008 \\ \hline
1400 & 0.25  & 0.0492 & 0.0011 \\ 
     & 0.5   & 0.0499 & 0.0008 \\ 
     & 2     & 0.0492 & 0.0009 \\ \hline
1500 & 0.25  & 0.0502 & 0.0010 \\ 
     & 0.5   & 0.0500 & 0.0010 \\ 
     & 2     & 0.0497 & 0.0010 \\ \hline
\end{tabular}
\end{table}
                The data was generated under the heavy-tailed setting, as described in Section \ref{section:heavy_tailed_setting}. The specific configuration is as follows:
                \begin{itemize}
                    \item \( p \) : We varied the number of features \( p \) from 100 to 1500 in increments of 100, to study the effect of increasing dimensionality on the model's performance.
                    \item \( n_{\text{train}} \) and \( n_{\text{test}} \) : For the training data, we used \( n = 200 \) samples, while the test data consisted of \( n_{\text{test}} = 1000 \) samples.
                    \item \( P_\mathrm{clust} \): Each component of $P_\mathrm{clust}$ is independently and identically distributed according to a generalized normal distribution, with specified location, scale, and shape parameters.
                    \begin{itemize}
                        \item The location parameter is $\mathbf{0}$.
                        \item The shape parameters \( \gamma \) are \( 0.25 \), \( 0.5 \), and \( 2 \)
                        \item The scale parameter \( \sigma \) is adjusted for each \( \gamma \) such that the variance is fixed at $1$.
                    \end{itemize} 
                    \item \( \mu \) : The mean vector \( \mu \) was set as \( \mu = \mathbf{1} \), meaning that all features had a common shift.
                    \item \( U \) : We applied an orthogonal transformation to the samples using a matrix \( U \), which was obtained from the QR decomposition of a randomly generated matrix \( A \). Each element of \( A \) was drawn from a standard normal distribution. The orthogonal matrix \( U \) is the result of the decomposition:
                    \[
                    A = UR
                    \]
                    where \( R \) is an upper triangular matrix.
                    \item \( \eta  \) : For each sample, we generated a label \( y \in \{-1, 1\} \) by multiplying a random scalar by a noise factor \( \eta \), where \( \eta = 0.05 \) in all experiments. 
                \end{itemize}
                
        \subsubsection{Model training}
            We used the maximum margin classifier, as described in Section \ref{section:maxmum_margin_algorithm}. The model was trained for $100000$ epochs to ensure convergence. The learning rate was set to $\beta=0.001$. Each experiment was repeated $50$ times, and the results were averaged. To ensure robustness, $95$\% confidence intervals were calculated based on the standard error of the mean.
        \subsection{Results}
            From Figure \ref{fig:gammma205025} and \ref{table:gamma_bo}, we can observe that the training error remains near zero across all dimensions, while the test error initially decreases and then stabilizes around the noise level as the dimension increases. This indicates that benign overfitting can occur even with distributions that have heavier tails than sub-gaussian distributions.
    \section{Experimental code and computing infrastructure}\label{appendix:infrastructure}
    
    The experimental code can be obtained from the anonymous URL on OSF:\url{https://osf.io/g6n9u/?view_only=f37a41efbee4421e8aa877f48c5879b4}
    
    The experiments were conducted in the following infrastructure:
    \begin{itemize}
        \item \textbf{GPU Type}: NVIDIA GeForce RTX 4090
        \item \textbf{Number of GPUs}: Single GPU
        \item \textbf{CPU Specifications}: 13th Gen Intel(R) Core(TM) i9-13900KF   3.00 GHz
        \item \textbf{Memory}: 32.0 GB 
        \item \textbf{Operating System}:  Windows 10 Home 23H2
        \item \textbf{Frameworks and Libraries}: The experiments for other figures were run using PyTorch, NumPy, SciPy, Matplotlib, Seaborn and Pandas on this infrastructure.
    \end{itemize}
\end{document}